\documentclass[twoside]{article}

\usepackage[preprint]{aistats2026}

\usepackage{hyperref}
\usepackage{url}

\usepackage[ruled,vlined]{algorithm2e} % Use algorithm2e package
\usepackage{comment}
\usepackage{amsmath}
\usepackage{xcolor}
\usepackage{graphicx}
\usepackage{amsfonts}
\usepackage{amsthm}
\usepackage{braket}
\usepackage{bm}
\usepackage{booktabs}
\newtheorem{theorem}{Theorem}
\newtheorem{lemma}{Lemma}
\newtheorem{corollary}{Corollary}

\newtheorem{definition}{Definition}
\newcommand{\NAMEA}{Q-ShiftDP\xspace}
\newcommand{\NAMEB}{Adaptive Q-ShiftDP\xspace}
\newtheorem{sublemma}{Lemma}[lemma]
\usepackage{subcaption}

\usepackage[round]{natbib}

\begin{document}
%\input{AC_responses}


\twocolumn[

\aistatstitle{Q-ShiftDP: A Differentially Private Parameter-Shift Rule for Quantum Machine Learning}

\aistatsauthor{
Hoang M. Ngo \And
Nhat Hoang-Xuan \And
Quan Nguyen 
}

\aistatsaddress{
University of Florida\\hoang.ngo@ufl.edu \And
University of Florida\\hoangx@ufl.edu \And
University of Florida\\quan.nguyen1@ufl.edu
}

\aistatsauthor{
Nguyen Do \And
Incheol Shin \And
My T. Thai
}

\aistatsaddress{
University of Florida\\nguyen.do@ufl.edu \And
Pukyong National University\\icshin@pknu.ac.kr \And
University of Florida\\mythai@cise.ufl.edu
}

]

\runningauthor{Hoang M. Ngo, Nhat Hoang-Xuan, Quan Nguyen, Nguyen Do, Incheol Shin, My T. Thai}

\begin{abstract}

Quantum Machine Learning (QML) promises significant computational advantages, but preserving training data privacy remains challenging. Classical approaches like differentially private stochastic gradient descent (DP-SGD) add noise to gradients but fail to exploit the unique properties of quantum gradient estimation. %In classical machine learning, differential privacy stochastic gradient descent (DP-SGD) is the standard approach that adds noise to model gradients to preserve privacy of training data. However, directly applying DP-SGD to QML is inefficient as it does not account for the distinctive properties of quantum gradient estimation in QML. 
In this work, we introduce the Differentially Private Parameter-Shift Rule (Q-ShiftDP), the first privacy mechanism tailored to QML. By leveraging the inherent boundedness and stochasticity of quantum gradients computed via the parameter-shift rule, Q-ShiftDP enables tighter sensitivity analysis and reduces noise requirements. We combine carefully calibrated Gaussian noise with intrinsic quantum noise to provide formal privacy and utility guarantees, and show that harnessing quantum noise further improves the privacy–utility trade-off. Experiments on benchmark datasets demonstrate that Q-ShiftDP consistently outperforms classical DP methods in QML.

  %Quantum Machine Learning (QML) holds the potential for significant computational advantages, but ensuring the privacy of training data in QML remains a critical challenge. In classical machine learning, differential privacy stochastic gradient descent (DP-SGD) is the standard approach which adds noise to model gradients to preserve privacy of training data. However, directly applying DP-SGD to QML is inefficient as it does not account for the distinctive properties of quantum gradient estimation in QML. In this work, we introduce the Differentially Private Parameter-Shift Rule (\NAMEA) which is the first privacy mechanism specifically designed to leverage the characteristics of quantum gradients computed via the parameter-shift rule (PSR) which is a widely used gradient estimation technique in QML. Unlike classical gradients, which are computed exactly via the chain rule, quantum gradients obtained through PSR are inherently bounded and stochastic. These properties enable for more precise sensitivity analysis and reduce the amount of noise required for privacy protection. Building on these advantages, we provide formal privacy and utility guarantees by combining carefully calibrated Gaussian noise with the intrinsic noise of quantum gradient estimation. Additionally, we demonstrate how quantum noise can be harnessed to enhance the privacy–utility trade-off. Experiments on various benchmark datasets show that \NAMEA consistently outperforms classical DP methods in QML, demonstrating the effectiveness of our approach.

\end{abstract}

\section{Introduction}

Quantum Machine Learning (QML) has emerged as a promising avenue for advancing machine learning by harnessing unique quantum phenomena such as superposition and entanglement. These quantum properties enable QML to learn high-dimensional and complex patterns that are intractable for classical machine learning (ML) algorithms~\cite{lloyd2013quantumalgorithmssupervisedunsupervised, Cai2015, Schuld2019}. As a result, QML potentially achieve advantages in learning tasks such as regression, classification and reinforcement learning.

Privacy remains a fundamental concern for all learning models, both classical and quantum. In classical ML, Differential Privacy (DP) has been established as the standard framework for providing rigorous privacy guarantees ~\cite{Dwork2006DifferentialP}. %By design, DP ensures that the output of a learning algorithm is minimally affected by the inclusion or removal of any single individual's data, thereby protecting the privacy of each participant~\cite{Dwork2006DifferentialP}.
Incorporating differential privacy (DP) into machine learning models typically involves injecting carefully calibrated noise into the learning process. %Noise can be added during the forward pass~\cite{Lcuyer2018CertifiedRT, FW_Phan2019ScalableDP, FW_Du_2023} (i.e., by perturbing the inputs or intermediate layer activations), or during the backward pass, where noise is added to the gradients~\cite{DPSGD,BW_Ghazi2025,BW_Zhang2025}. 
Among %these approaches
recent related work, gradient perturbation has been shown to provide the best trade-off between privacy and utility in deep learning models~\cite{Jarin2022_InOutGradientPerturb}. Notable ones are DP-SGD~\cite{DPSGD}, adaptive clipping~\cite{DPSGD_Thakkar2019DifferentiallyPL}, and personalized privacy budgets~\cite{DPSGD_Boenish2023}. %The most widely used method, DP-SGD~\cite{DPSGD}, first clips per-example gradients to bound their sensitivity and then adds Gaussian noise to the averaged gradient before updating the model. Several variants, such as adaptive clipping~\cite{DPSGD_Thakkar2019DifferentiallyPL}, importance sampling~\cite{DPSGD_Wei2022DPISAE}, and personalized privacy budgets~\cite{DPSGD_Boenish2023}, further enhance model performance while maintaining rigorous DP guarantees.

Recent studies have extended the concept of DP to the quantum domain, giving rise to Quantum Differential Privacy (QDP)~\cite{ZhouQDPfirst,Hirche2023InfoTheoryQDP}. This line of work is particularly important for QML, as QML models are designed to capture highly expressive, high-dimensional representations that may inadvertently encode sensitive information about individual training samples~\cite{Rajawat2023_app,Thakkar2024_app,Marengo2025_app, Bukkarayasamudram2025_app}. Thus, rigorously enforcing DP in QML is a practical prerequisite for deploying quantum learning algorithms in real-world, data-sensitive applications.
%Most existing works in QDP focus on injecting quantum noise during the forward pass~\cite{Du2021Qnoise,watkins2023quantum,song2025towards}, analogous to input or activation perturbation in classical machine learning. 
However, as in classical settings, this approach can suffer from significant utility degradation. 
%To improve the privacy-utility trade-off, some works~\cite{YenChi2021,YenChi2025} explore applying DP-SGD heuristically to perturb the gradients of QML models. Nonetheless, 
There are two key differences in quantum gradients that make direct application of DP-SGD inefficient in QML. First, quantum gradients are estimated statistically via the parameter-shift rule~\cite{Wierichs2022generalparameter}, rather than being exactly computed. This inherent stochasticity can itself serve as a source of privacy, thereby reducing the need for additional Gaussian noise. Second, unlike classical gradients which are unbounded, quantum gradients are naturally bounded due to the properties of quantum operators and observables. Studying this bound allows for precise sensitivity analysis and avoids unnecessary gradient clipping, which can reduce model performance. 

In this work, we propose \NAMEA, a Differentially Private Parameter-Shift Rule (PSR) mechanism designed to protect the privacy of QML models. To our knowledge, this is the first approach that explicitly exploits the unique characteristics of quantum gradients computed via PSR to efficiently protect the training process in QML. Specifically, we provide a rigorous theoretical analysis of the $\ell_2$-sensitivity of quantum gradients. Based on this, we establish both privacy and utility guarantees for \NAMEA by combining the intrinsic variance of estimated quantum gradients with carefully calibrated Gaussian noise. Furthermore, %we investigate the variance of quantum gradient estimators under both noisy and ideal quantum environments to gain insights into how the amount of added Gaussian noise can be reduced without compromising privacy.
we demonstrate how the privacy-utility trade-off of \NAMEA can be further improved in two key ways: first, by treating physical depolarizing noise as a privacy-enhancing resource, and second, by introducing an adaptive technique to tailor the amount of artificial noise based on empirically estimated per-sample variance. To assess the effectiveness of \NAMEA, we conduct experiments on three benchmark datasets:  Bars \& Stripes, Binary Blobs and Downscaled MNIST \cite{bowles2024better,recio2025train}. The empirical results demonstrate that \NAMEA consistently achieves higher utility than input-perturbation based DP mechanisms in QML.

%\textbf{Contributions.} Our main contributions are summarized as follows:

%\begin{itemize}
%    \item We propose \NAMEA, the first differentially private mechanism tailored for quantum machine learning that leverages the unique statistical and bounded properties of quantum gradients computed via the parameter-shift rule.
%\item We provide a rigorous theoretical analysis of the $\ell_2$-sensitivity of quantum gradients, and then establish formal privacy and utility guarantees for \NAMEA.

%\item We analyze the variance of quantum gradient estimators under both noisy and ideal quantum environments to provide insights into efficient noise calibration for privacy.
%\item We conduct extensive experiments on three benchmark datasets: [dataset names], demonstrating that \NAMEA consistently achieves higher utility compared to forward-pass DP mechanisms.
%\end{itemize}

\section{Backgrounds}
%In this section, we review the essential background on quantum information and variational quantum machine learning (VQML). We introduce quantum states, operations, and measurements with an emphasis on expectation values and shot noise. We then describe the framework of VQML with a focus on parameterized quantum circuits (PQCs) and the parameter-shift rule for gradient evaluation. Additional backgrounds are included in Appendix~\ref{sec:add_background}.
In this section, we first review quantum states, operations, and measurements with an emphasis on expectation values and shot noise. We then describe variational quantum machine learning (VQML) and the parameter-shift rule (PSR) for gradient estimation. Finally, we summarize differential privacy and classical DP-SGD, and illustrate why a direct transfer to QML is suboptimal. Additional material and related works appear in Appendices~\ref{sec:add_background} and~\ref{sec:related_works}.

\subsection{Fundamentals of Quantum Information}
%\mt{Do we need QM and Observation for readers to understand the main paper? Can we move to Appendix?}
\noindent \textbf{Quantum States and Operations.} Quantum computing systems are built upon qubits which are the fundamental units of quantum information. The state of a qubit can exist in the superposition of $\ket{0}$ and $\ket{1}$. For example, a single-qubit pure state can be written as $\ket{\psi} = a\ket{0}+b\ket{1}$ where $a,b \in \mathbb{C}$ and the normalization condition $|a|^2+|b|^2 = 1$ holds. Multi-qubit systems of $n$ qubits occupy a $2^n$-dimensional Hilbert space formed by the tensor product of individual states.

\noindent \textbf{Quantum Measurement and Observables.} A measurement in a quantum system extracts classical information from a quantum state $\ket{\psi}$ by probing a physical property, such as energy, position, or spin. Each such property, known as an \emph{observable}, is represented mathematically by a Hermitian operator $O$.

The possible outcomes of measuring this observable are its eigenvalues $\{\lambda_i\}$. Based on spectral theorem~\cite{moretti2020spectral}, an observable $O$ can be decomposed in terms of its projectors as $O = \sum_i \lambda_i P_i$
%\nguyendo{should be inline equation}
where $P_i$ is the projection operator onto the eigenspace corresponding to $\lambda_i$. These projection operators satisfy $\sum_i P_i = I$.
The expectation value of the observable $O$, denoted as $\langle O \rangle$ represents the statistical average outcome of many identical measurements performed on identically prepared systems. This average is calculated by summing each possible outcome weighted by its probability:
\begin{equation}
    \begin{aligned}
        \langle O \rangle &= \sum_i \lambda_i \Pr(\lambda_i)= \sum_i \lambda_i \bra{\psi} P_i \ket{\psi}\\  
        &= \bra{\psi} \left( \sum_i \lambda_i P_i \right) \ket{\psi} = \bra{\psi} O \ket{\psi}.
    \end{aligned}
\end{equation}

Importantly, $\langle O \rangle$ is always bounded by the minimum and maximum eigenvalues of $O$, a property that later enables tight sensitivity analysis.

\noindent\textbf{Finite-Shot Estimation and Shot Noise.}

In practice, the expectation value $\langle O \rangle$ cannot be accessed exactly. Instead, it is estimated using $N_s$ independent measurement \emph{shots}. A single shot yields one of its eigenvalues $\lambda_i$. To obtain statistical information, we perform $N_s$ independent shots on identically prepared copies of the state $\ket{\psi}$. Each shot produces an outcome $\hat{a} \in \{\lambda_i\}$ with probability $  \Pr(\hat{a} = \lambda_i) = \bra{\psi} P_i \ket{\psi}$.
%\nguyendo{should be inline equation}
After $N_s$ shots, we collect $N_s$ independent outcomes $\{\hat{a}_1, \dots, \hat{a}_{N_s}\}$. The empirical estimator of the expectation value is then
$
    \bar{A}_{N_s} = \frac{1}{N_s}\sum_{i=1}^{N_s} \hat{a}_i
$.
The variance of this estimator decreases with the number of shots and is given by
\begin{equation}
    \operatorname{Var}[\bar{A}_{N_s}] = \frac{\sigma_{\text{shot}}^2}{N_s}
\end{equation}

%\nguyendo{should be inline equation}

where $\sigma_{\text{shot}}^2$ is the variance of a single-shot outcome distribution. This intrinsic statistical noise, known as \emph{shot noise}, plays a central role in both optimization dynamics and privacy guarantees in QML.

\subsection{Variational Quantum Machine Learning}
Variational Quantum Machine Learning (VQML) is a quantum-classical algorithm in which a quantum circuit with tunable parameters is optimized using classical techniques. This approach is well-suited for noisy intermediate-scale quantum (NISQ) devices, as it harnesses the expressive power of quantum computing to learn complex patterns while relying on efficient classical optimizers to overcome the limitations of current quantum hardware. In this section, we first introduce Parameterized Quantum Circuits (PQCs), which serve as the cornerstone of VQML. We then present a technique for computing gradients of PQCs, called Parameter-Shift Rule.

\noindent \textbf{Parameterized Quantum Circuit (PQC) Models.}
A VQML model is trained by defining a cost function $C(\boldsymbol{\theta})$ that depends on the parameters $\boldsymbol{\theta}$ of a PQC, and then minimizing it through iterative updates. A PQC typically consists of three components: (i) an encoding that embeds classical data $x$ into a quantum state, (ii) a variational ansatz $U(\boldsymbol{\theta})$ with tunable parameters $\boldsymbol{\theta}$, and (iii) measurement operators $\{\hat{O}_i\}$ that extract classical information from the quantum state. The variational ansatz can be expressed as a unitary transformation
 %The full PQC transformation can be written as follow:
%$$
%\ket{\psi(x, \boldsymbol{\theta})} = U(\boldsymbol{\theta}) U_{\text{enc}}(x) \ket{0}^{\otimes n}.
%$$
%The variational ansatz can be represented as a unitary model as:
$
    U(\boldsymbol{\theta}) = \prod_{k=1}^K e^{i \theta_k G_k}
$
where $\boldsymbol{\theta} = (\theta_1, \ldots, \theta_K)$ are trainable parameters and $\{G_k\}$ are Hermitian generators. Given an input $x$ encoded into a quantum state $\ket{\psi_0(x)}$, the PQC produces:
\begin{equation}
    \ket{\psi(\boldsymbol{\theta}, x)} = \hat{U}(\boldsymbol{\theta}) \ket{\psi_0(x)}
\end{equation}

Finally, the cost function is computed by measuring an observable $\hat{O}_y$ associated with label $y$, defined as:
\begin{equation}
    C(\boldsymbol{\theta}) = \mathbb{E}_{(x,y)\sim S} \left[ \langle \psi(\boldsymbol{\theta},x)| \hat{O}_y |\psi(\boldsymbol{\theta},x)\rangle \right]
\end{equation}
%The expectation value of an observable $O$ provides the output:
%$$
%f(x, \boldsymbol{\theta}) = \bra{\psi(x, \boldsymbol{\theta})} O \ket{\psi(x, \boldsymbol{\theta})} 
%= \bra{0}^{\otimes n} U_{\text{enc}}^\dagger(x) U^\dagger(\boldsymbol{\theta}) O U(\boldsymbol{\theta}) U_{\text{enc}}(x) \ket{0}^{\otimes n}.
%$$

\noindent \textbf{Gradient Evaluation via Parameter-Shift Rule.} Computing the gradients of $C(\boldsymbol{\theta})$ is essential for training VQML models. However, unlike classical ML, direct backpropagation is not applicable to PQCs because quantum measurements are inherently probabilistic and intermediate quantum states cannot be accessed or copied due to the no-cloning theorem. Instead, gradients are estimated using the a technique, called parameter-shift rule~\cite{Li2017,Mitarai2018,Wierichs2022generalparameter}.

Consider a variational ansatz $U(\boldsymbol{\theta}) = \prod_{k=1}^K e^{i \theta_k G_k}$ where $G_k$ has only two eigenvalues. Let the difference between two eigenvalues be $\Omega_k$. We can calculate the deviation of $f(\boldsymbol{\theta}) = \langle \psi(\boldsymbol{\theta},x)| \hat{O}_y |\psi(\boldsymbol{\theta},x)\rangle$ w.r.t $\theta_k$ by parameter-shift rule as
\begin{equation}
    \frac{\partial f}{\partial \theta_k}
= \frac{\Omega_k}{2\sin{(\Omega_ks)}}\Big[f\!\left(\theta_k + s\right) - f\!\left(\theta_k - s\right)\Big]
\end{equation}
where $s$ is freely chosen from $(0,\pi)$. A common choice in practice is $s = \pi/(2\Omega_k)$, which simplifies the formula because $\sin(\Omega_k s) = 1$. This approach provides an unbiased gradient evaluation while requiring only two evaluations of the quantum circuit per parameter. $f\!\left(\theta_k + s\right)$ and $f\!\left(\theta_k - s\right)$ are practically implemented by a positive-shifted and a negative-shifted circuits.

\subsection{Differential Privacy}

\noindent \textbf{Foundations of Differential Privacy.}
Differential Privacy (DP) provides a rigorous framework that limits how much influence a single data point can have on the output of an algorithm~\cite{Dwork2006DifferentialP}. 

\begin{definition}
    [$(\varepsilon, \delta)$-Differential Privacy] A randomized mechanism $\mathcal{M}:\mathcal{D} \rightarrow \mathcal{R}$ satisfies $(\varepsilon, \delta)$-differential privacy if for any two adjacient datasets $D_1$ and $D_2$ that differs by a single element, and for any subset of outputs $S \subseteq \mathcal{R}$, the following inequality holds:
    \begin{equation}
        \operatorname{Pr}[\mathcal{M}(D_1) \in S] \leq e^\varepsilon\operatorname{Pr}[\mathcal{M}(D_2) \in S] + \delta
    \end{equation}
\end{definition}

In this definition, $\varepsilon \geq 0$ quantifies the maximum privacy loss. Specifically, it controls how much the output distribution may change when one data point is modified. In addition, the parameter $\delta \in [0,1)$ allows for a small probability of failure in this guarantee. Smaller values of $\varepsilon$ and $\delta$ lead to stronger privacy protection.

\noindent \textbf{Differentially Private Stochastic Gradient Descent (DP-SGD).}
The dominant approach for private training in classical ML is DP-SGD. Differentially Private SGD (DP-SGD)~\cite{DPSGD}. DP-SGD enforces privacy by modifying the standard SGD update through two key mechanisms:
\begin{enumerate}
    \item \emph{Gradient clipping:} Each per-sample gradient is clipped to a fixed norm bound $C$. This ensures that sensitivity of the gradients is bounded.
    \item \emph{Noise addition:} After clipping, Gaussian noise with variance calibrated to $C$, $\varepsilon$, and $\delta$ is added to the aggregated gradient. The noise masks the exact contribution of each individual, thereby ensuring differential privacy.  
\end{enumerate}  
By carefully balancing the clipping bound and noise scale, DP-SGD achieves a trade-off between privacy and model accuracy. Privacy guarantees are tracked across training iterations using a privacy accountant, which accumulates the total privacy loss $(\varepsilon, \delta)$ over the course of training.

In contrast, gradients in VQML, computed via the PSR, are inherently bounded due to the finite spectrum of quantum observables and are already stochastic because of finite-shot measurement noise. As we show next, exploiting these structural properties enables tighter sensitivity analysis and substantially reduces the amount of additional noise required to achieve DP.

\section{Differentially Private Parameter-Shift Rule (\NAMEA)}

This section details our central contribution, the \NAMEA algorithm, which is a differentially private algorithm to protect the privacy of the quantum training process, and its theoretical underpinnings. First, we present the core algorithm and establish theoretical guarantees. This includes the main privacy guarantee, that illustrates how the inherent shot noise from measurements can be used to reduce the need for externally added artificial noise, and the utility analysis. Next, we extend this analysis to noisy quantum environments where physical depolarizing noise can be a source for privacy by providing a guaranteed, non-trivial lower bound on the single-shot variance. Finally, we show how to empirically estimate the single-shot noise variance by measurement outcomes. That allows \NAMEA to dynamically tailor the amount of artificial noise added. Due to the page limit, all proofs in this paper are included in Appendix~\ref{sec:proofs}.

%demonstrate how the inherent statistical shot noise from quantum measurements can be treated as a privacy-enhancing resource, thus reducing the amount of artificial noise required.
%analyze how inherent quantum noise contributes to the overall privacy loss. Finally, we examine the effect of shot noise and demonstrate how it can be leveraged to enhance model utility.

\begin{figure}[htp]
    \centering
    \includegraphics[width=1.0\linewidth]{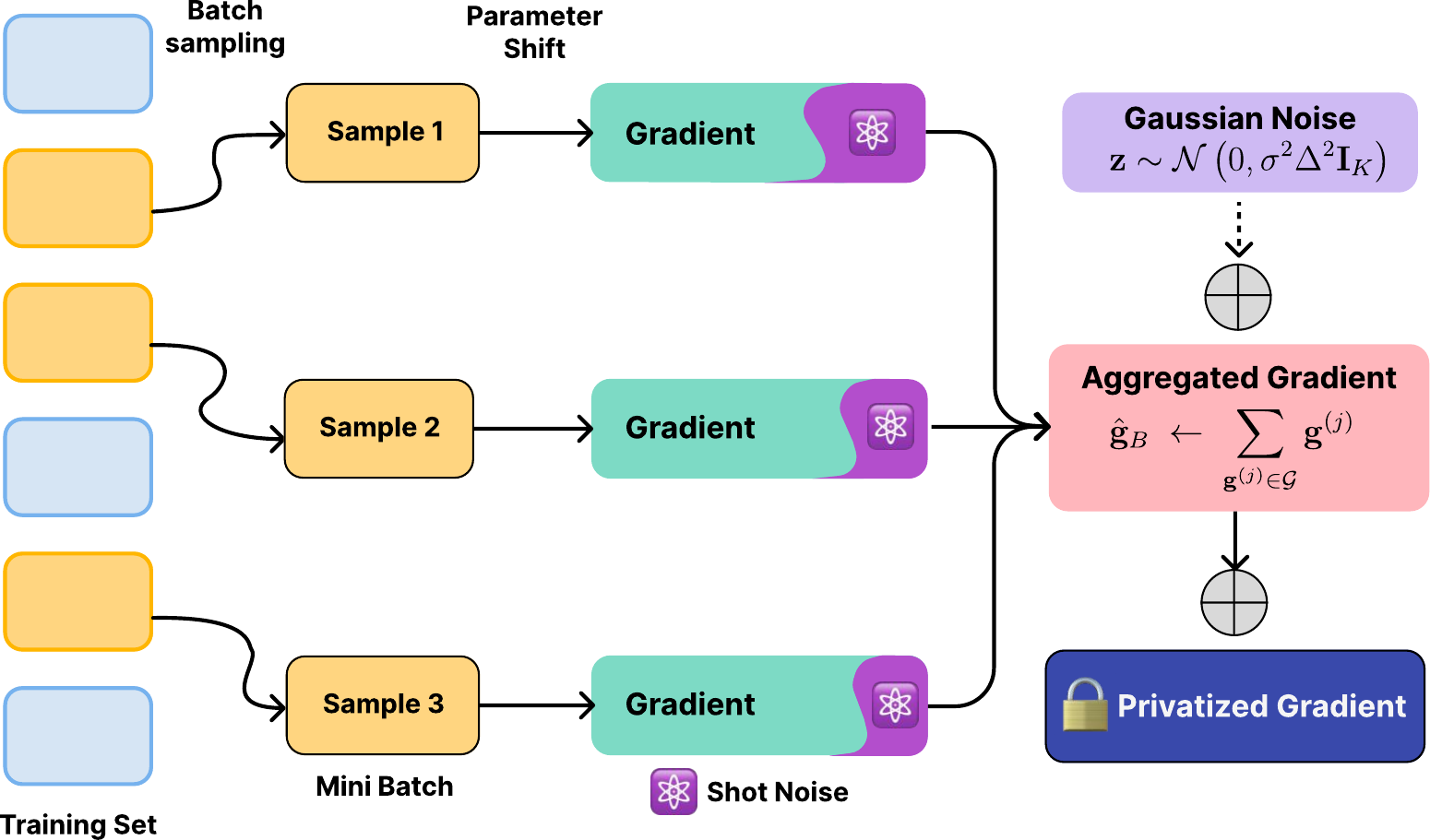}    
    \caption{Overall framework of Q-ShiftDP. Each mini-batch is processed via the parameter-shift rule, yielding inherently noisy and bounded gradients. %$\tfrac{\partial f}{\partial \theta_k} = \tfrac{\Omega_k}{2},[f(\theta_k+s)-f(\theta_k-s)]$, where $s=\pi/(2\Omega_k)$. 
    %resulting gradients are inherently noisy and bounded. 
    The aggregated gradient is then perturbed with Gaussian noise $z \sim \mathcal{N}(0,\sigma^2\Delta^2 \mathbf{I}_K)$ to ensure $(\varepsilon,\delta)$-DP.}
    %The aggregated gradient $\hat g_B = \sum_{g^{(j)}\in\mathcal{G}} g^{(j)})$ is then perturbed with Gaussian noise $z \sim \mathcal{N}(0,\sigma^2\Delta^2 I_d)$ to achieve $(\varepsilon,\delta)$-differential privacy, resulting in the privatized gradient used for updating model parameters.}
    \label{fig:overall-framework}
\end{figure}

\subsection{Proposed Algorithm}

% \begin{figure*}[htp]
%     \centering
% \includegraphics[width=0.9\linewidth]{figs/Q-ShiftDP.pdf}    
% \caption{Overall framework of Q-ShiftDP.}
%     \label{fig:overall-framework}
% \end{figure*}

%\mt{Somewhere need to say all the proofs aer in Appendix?}

To design an efficient differentially private algorithm for training QML models, it is essential to account for the unique characteristics of QML training, which differ from those in classical machine learning. Two key distinctions are: (i) quantum gradients are inherently bounded, and (ii) gradients are estimated through statistical sampling rather than computed exactly.

We first investigate the boundedness of gradients in QML. According to PSR, gradients are computed as the difference between two shifted expectation values. Thus, it is crucial to understand the range of possible expectation values produced by a quantum circuit. Lemma~\ref{lemma:eigenvalues} shows that the expectation value is always bounded by the extremal eigenvalues of the observable.

\begin{lemma}
    Given the objective function $C(\boldsymbol{\theta}) = \langle\psi_0|\hat{U}(\boldsymbol{\theta})^{\dagger}\hat{O}\hat{U}(\boldsymbol{\theta})|\psi_0\rangle$ where $\hat{O}$ is an observable with a minimum and maximum eigenvalues $\lambda_{\text{min}}$ and $\lambda_{\text{max}}$ respectively and $\hat{U}(\boldsymbol{\theta})$ is a parametric unitary. The value of $C(\boldsymbol{\theta})$ is bounded by these same extremal eigenvalues for all possible parameter values $\boldsymbol{\theta}$.
    \begin{equation}
        \lambda_{\min} \le C(\boldsymbol{\theta}) \le \lambda_{\max}
    \end{equation}
\label{lemma:eigenvalues}
\end{lemma}

%\nguyendo{should be inline equation}

Next, we establish a bound on the $\ell_2$-sensitivity of the gradients, as stated in Lemma~\ref{lemma:L2_sensitivity}. The bound $\Delta$ is determined by the extremal eigenvalues of the observable and the frequencies ${\Omega_k}$ of the generators $G_k$, which can be specified prior to training. Thus, the $\ell_2$-sensitivity of the gradients can be controlled directly, without the need for gradient clipping. Based on this property, we design a differentially private algorithm for QML training, presented in Algorithm~\ref{alg:dp-prs}.

\begin{lemma}[$\ell_2$-Sensitivity of the Quantum Gradient]
    We define an objective function $C(\boldsymbol{\theta}) = \Braket{\psi_0(x) | \hat{U}(\boldsymbol{\theta})^{\dagger} \hat{O} \hat{U}(\boldsymbol{\theta}) | \psi_0(x)}$ where the observable $\hat{O}$ has eigenvalues in range from $\lambda_{\min}$ to $\lambda_{\max}$ and parametric unitary $\hat{U}(\boldsymbol{\theta}) = \prod_{k=1}^{K} e^{i\theta_k G_k}$ with frequencies $\Omega_k$. 
Suppose the gradient of $C(\boldsymbol{\theta})$ with respect to parameter $\theta_k$ is computed using the parameter-shift rule:
\begin{equation}
        \mathbf{g}_k = \tfrac{\Omega_k}{2}\,(r_+ - r_-),
\end{equation}
%\nguyendo{should be inline equation}
where $r_+$ and $r_-$ are expectation values of the observable under a shifted amount $s = \frac{\pi}{2\Omega_k}$.  
    The \textbf{$\ell_2$-norm} of this gradient is strictly bounded by the sensitivity $\Delta$:
    \begin{equation*}
        \|\mathbf{g}_k\|_2 \le \Delta \quad \text{where} \quad \Delta = \frac{\lambda_{\max} - \lambda_{\min}}{2} \sqrt{\sum_{k=1}^{K} \Omega_k^2}
    \end{equation*}
    \label{lemma:L2_sensitivity}
\end{lemma}

Algorithm~\ref{alg:dp-prs} takes as input: a data encoding map $x \mapsto \ket{\psi_0(x)}$, a parametric unitary $\hat{U}(\boldsymbol{\theta}) = \prod_{k=1}^{K} e^{i\theta_k G_k}$ with generator frequencies $\{\Omega_k\}$, a set of observables $\{\hat{O}_y\}$ whose eigenvalues lie in $[\lambda_{\min}, \lambda_{\max}]$, a private dataset $S$, as well as training hyperparameters including the learning rate $\text{lr}$, the number of steps $T$, batch size $B$, and the noise multiplier $\sigma$.  The algorithm starts by computing the theoretical $\ell_2$-sensitivity bound $\Delta$ of the gradients based on Lemma~\ref{lemma:L2_sensitivity}. For each training step, a mini-batch $B$ is sampled from the private dataset (for simplicity, we use $B$ to denote both the mini-batch and its size). For each data point $(x_j,y_j) \in B$, we initialize a gradient vector $\mathbf{g}^{(j)}$. Then, for every parameter $\theta_k$, the $k$-th component of the gradient is estimated using the parameter-shift rule: first, the shift amount $s = \pi/(2\Omega_k)$ is set; next, the intermediate state $|\phi\rangle$ and operator $\hat{A}$ are defined. Two shifted circuits, including the positive-shifted circuit $\hat{A} \hat{U}_k(\theta_k+s)|\phi\rangle$ and the negative-shifted circuit $\hat{A} \hat{U}_k(\theta_k-s)|\phi\rangle$, are then executed in $N_s$ shots, each to obtain empirical estimates $r_+$ and $r_-$. The gradient component is computed as $\mathbf{g}^{(j)}_k = \tfrac{\Omega_k}{2}(r_+ - r_-)$. After all parameters are processed, $\mathbf{g}^{(j)}$ is appended to the gradient list $\mathcal{G}$. Once all samples in the batch are processed, the gradients are averaged, Gaussian noise $\mathbf{z} \sim \mathcal{N}(0,\sigma^2 \Delta^2 \mathbf{I}_K)$ is added, and the noisy gradient $\tilde{\mathbf{g}}$ is used to update the parameters with learning rate $\textbf{\text{lr}}$. After $T$ iterations, the algorithm outputs the differentially private parameters $\boldsymbol{\theta}_T$. 

Lemma~\ref{lemma:iterationDP} establishes the per-iteration privacy guarantee of Algorithm~\ref{alg:dp-prs}. In this context, we account for the second key distinction between quantum and classical training: quantum gradients are estimated statistically. The statistical noise inherent in quantum measurements, denoted as \emph{shot noise}, can serve as an additional source of privacy. Shot noise is given by the single-shot variance divided by the number of shots, $N_s$. By the central limit theorem, when the number of shots $N_s$ is sufficiently large, the estimated gradient components follow a Gaussian distribution with variance $\sigma_{\mathrm{shot}}^2 / N_s$, where $\sigma_{\mathrm{shot}}^2$ represents a guaranteed lower bound on the single-shot variance across all possible inputs in all positive-shifted and negative-shifted circuits. Shot noise combines with the injected Gaussian noise, which is calibrated based on the $\ell_2$-sensitivity, to yield the final guarantee.

\begin{lemma}
Suppose $N_s$ is large enough so that the finite-shot estimators $\hat{r}_+$ and $\hat{r}_-$ are independent random variables drawn from Gaussian distributions with the variance at least $\sigma_{\mathrm{shot}}^2/N_s$. Then, Algorithm~\ref{alg:dp-prs} ensures $(\varepsilon_0, \delta_0)-\operatorname{DP}$ per iteration with:
\begin{equation}
    \varepsilon_0 = \frac{\sqrt{2 \ln \frac{1.25}{\delta_0}}}{\sqrt{\frac{ 2B\sigma_{\mathrm{shot}}^2}{N_s(\lambda_{\max} - \lambda_{\min})^2} + \sigma^2}}
\end{equation}
\label{lemma:iterationDP}
\end{lemma}

We emphasize that Lemma~\ref{lemma:iterationDP} characterizes the per-iteration privacy and does not account for mini-batch subsampling or the cumulative privacy loss over $T$ training steps. In Theorem~\ref{thm:main}, we extend this result to provide a lower bound on the artificial Gaussian noise required to achieve a desired overall privacy budget. Importantly, the inherent shot noise from finite-shot gradient estimation reduces the need for additional artificial noise: the noisier the estimation, the less extra artificial noise must be injected. We note that if gradients were computed exactly (i.e., $\sigma_{\mathrm{shot}}^2 / N_s \approx 0$), the amount of required artificial noise is equivalent with that of DP-SGD for classical ML~\cite{DPSGD}.

\begin{theorem}
    Consider Algorithm~\ref{alg:dp-prs} (\NAMEA) applied to a private dataset $S$ of size $N$, with mini-batch size $B$, sampling rate $q = B/N$, $T$ training iterations, and per-sample $\ell_2$-gradient sensitivity $\Delta$ (from Lemma~\ref{lemma:L2_sensitivity}).
    Let the desired total privacy guarantee be $(\varepsilon, \delta)$, for any $\delta \in (0,1)$ and for a sufficiently small $\varepsilon$ (specifically, $\varepsilon < c_1 q^2 T$ for a universal constant $c_1$). The algorithm guarantees $(\varepsilon, \delta)$-DP if the chosen explicit noise multiplier $\sigma$ satisfies:
    \begin{equation}
        \sigma^2 \ge  \max\left(0, C_{\text{DP}} - \frac{2B \sigma_{\mathrm{shot}}^2}{N_s (\lambda_{\max}-\lambda_{\min})^2} \right)
    \end{equation}
    where $C_{\text{DP}} = \left( c_2 \frac{q \sqrt{T \log(1/\delta)}}{\varepsilon} \right)^2$ with $c_2$ is a constant.
    \label{thm:main}
\end{theorem}

%\nguyendo{I should denote 2 complex formula terms as 2 auxiliary variables to avoid margin violation}

Finally, we analyze the utility of \NAMEA by quantifying the discrepancy between the statistically estimated gradient $\tilde{\mathbf{g}}$ (obtained via subsampling and measurement) and the true dataset gradient $\nabla C_S(\boldsymbol{\theta})$. Theorem~\ref{thm:utility} provides an upper bound on the expected squared $\ell_2$ error of this estimate. The bound is derived from three sources of errors: (i) sampling error due to mini-batch selection, (ii) statistical error from finite-shot measurements, and (iii) variance from the artificial noise added. An interesting aspect of the bound in Theorem~\ref{thm:utility} is the dependence on the number of shots $N_s$. Increasing $N_s$ improves utility by reducing the shot error in gradient estimation, thereby tightening the bound. However, this comes with a trade-off: larger $N_s$ diminishes the inherent randomness of quantum measurements, which contributes to privacy. Thus, more artificial noise is required to maintain the same privacy guarantee.

\begin{theorem}[Bound on the Gradient Estimation Error for \NAMEA]
    Let $\nabla C_S(\boldsymbol{\theta})$ be the true gradient over the full dataset $S$, and let $\tilde{\mathbf{g}}$ be the final noisy gradient estimate produced by Algorithm 1 (\NAMEA) for a mini-batch of size $B$. The Mean Squared Error (MSE) of the estimator is bounded by:
    \begin{equation}
        \mathbb{E}\left[ \|\tilde{\mathbf{g}} - \nabla C_S(\boldsymbol{\theta})\|^2 \right] \le \frac{\Delta^2}{B} \left( 1 + \frac{1}{2N_s} \right) + \frac{K \sigma^2 \Delta^2}{B^2}
    \end{equation}
    where $\Delta$ is the per-sample $\ell_2$-sensitivity from Lemma~\ref{lemma:L2_sensitivity}, $N_s$ is the number of measurement shots, $K$ is the number of parameters, and $\sigma$ is the artificial noise multiplier.
    \label{thm:utility}
\end{theorem}

In \NAMEA, the guaranteed lower bound on the single-shot variance $\sigma^2_{\text{shot}}$ plays a crucial role in determining the required amount of artificial noise $\sigma^2$. However, without additional information about the positive and negative-shifted circuits, establishing a meaningful lower bound is challenging. In next section, we derive a non-trivial lower bound $\sigma^2_{\text{shot}}$ in the presence of depolarizing noise in the quantum circuits.

\begin{algorithm*}[h]
\DontPrintSemicolon
\caption{Differentially Private Parameter-Shift Rule (\NAMEA)}
\label{alg:dp-prs}
\KwIn{
    Data encoding $x \mapsto \ket{\psi_0(x)}$,
    Parametric unitary $\hat{U}(\boldsymbol{\theta}) = \prod_{k=1}^{K} e^{i\theta_k G_k}$ with frequencies $\Omega_k$,
    Set of observables $\{\hat{O}_y\}$ with eigenvalues in $[\lambda_{\min}, \lambda_{\max}]$,
    Private dataset $S$,
    Learning rate $\textbf{\text{lr}}$, NumSteps $T$, Batch size $B$,
    Noise multiplier $\sigma$.
}
\KwOut{Trained parameters $\boldsymbol{\theta}_{T}$.}
\BlankLine
\tcp{Define sensitivity based on the theoretical maximum L2 norm of the quantum gradient}
$\Delta \leftarrow \frac{\lambda_{\max} - \lambda_{\min}}{2} \sqrt{\sum_{k=1}^{K} \Omega_k^2}$\;

Initialize parameters $\boldsymbol{\theta}_0$ randomly\;
\For{$t \leftarrow 0$ \KwTo $T-1$}{
    Sample a mini-batch $B \subset S$\;
    Initialize an empty list for gradient estimates: $\mathcal{G} \leftarrow []$\;
    \ForEach{sample $(x_j, y_j) \in B$}{
        %\tcp{Compute the quantum gradient vector}
        
        Initialize vector: $\mathbf{g}^{(j)} \leftarrow \mathbf{0} \in \mathbb{R}^K$\;

        %Prepare the observable $\hat{O} = |y_j\rangle \langle y_j|$
        
        \For{$k \leftarrow 1$ \KwTo $K$}{
            %\tcp{Compute k-th component of the gradient, $\partial C / \partial \theta_k$}
            
            Define shift amount: $s \leftarrow \pi / (2\Omega_k)$\;
            
            %\tcp{Define circuit components}
            $|\phi\rangle \leftarrow (\prod_{l=1}^{k-1} \hat{U}_l(\theta_l)) |\psi_0(x_j)\rangle$\ \tcp*{Define intermediate state}
            $\hat{A} \leftarrow (\prod_{l=k+1}^{K} \hat{U}_l(\theta_l))^\dagger \hat{O}_{y_j} (\prod_{l=k+1}^{K} \hat{U}_l(\theta_l))$\ \tcp*{Define intermediate operator}

            %\tcp{Evaluate shifted expectation values}
            %\tcp{Evaluate shifted expectation values by sampling}
            
            Run positive-shifted circuit $\hat{A} \hat{U}_k(\theta_k+s)|\phi\rangle$ with $N_s$ times, collect outcomes $\{\hat{a}_{j,k,i}\}_{i=1}^{N_s}$\;
            $r_+ \leftarrow \frac{1}{N_s}\sum_{i=1}^{N_s} \hat{a}_{j,k,i}$\ \tcp*{Estimator for positive shift}
            
            Run negative-shifted circuit $\hat{A} \hat{U}_k(\theta_k-s)|\phi\rangle$ with $N_s$ times, collect outcomes $\{\hat{b}_{j,k,i}\}_{i=1}^{N_s}$\;
            $r_- \leftarrow \frac{1}{N_s}\sum_{i=1}^{N_s} \hat{b}_{j,k,i}$\ \tcp*{Estimator for negative shift}
            %$r_+ \leftarrow \langle\phi|\hat{U}_k(\theta_k+s)^\dagger \hat{A} \hat{U}_k(\theta_k+s)|\phi\rangle$\ \tcp*{Estimate with $N_s$ shots}
            
            %$r_- \leftarrow \langle\phi|\hat{U}_k(\theta_k-s)^\dagger \hat{A} \hat{U}_k(\theta_k-s)|\phi\rangle$\ \tcp*{Estimate with $N_s$ shots}
            
            $\mathbf{g}^{(j)}_k \leftarrow \frac{\Omega_k}{2} (r_+ - r_-)$ \tcp*{Parameter-shift rule estimator}

            %$\mathbf{g}^{(j)}_k \leftarrow \bar{\mathbf{g}}^{(j)}_k / \max\left(1, \frac{|\bar{\mathbf{g}}^{(j)}_k|}{C_{clip}}\right)$\ \tcp*{Clip the estimator}
            
        }
        %\tcp{Clip the L2 norm of the gradient estimate}
        %$\bar{\mathbf{g}}_j \leftarrow \mathbf{g}_j / \max\left(1, \frac{\|\mathbf{g}_j\|_2}{C_{clip}}\right)$\;
        Append $\mathbf{g}^{(j)}$ to $\mathcal{G}$\;
    }
    %\tcp{Aggregate, add noise, and update}
    $\hat{\mathbf{g}}_{B} \leftarrow \sum_{\mathbf{g}^{(j)} \in \mathcal{G}} \mathbf{g}^{(j)}$\;
    Sample noise: $\mathbf{z} \sim \mathcal{N}(0, \sigma^2 \Delta^2 \mathbf{I}_K)$;
    $\tilde{\mathbf{g}} \leftarrow \frac{1}{B} \left( \hat{\mathbf{g}}_{B} + \mathbf{z} \right)$ ;
    $\boldsymbol{\theta}_{t+1} \leftarrow \boldsymbol{\theta}_t - \textbf{\text{lr}} \cdot \tilde{\mathbf{g}}$\;
}
\Return $\boldsymbol{\theta}_T$
    \end{algorithm*}

\subsection{Single-Shot Measurement Variance under Noisy Environment}
\label{sec:quantum_noise_error_bound}

%In this section, we first relate the single-shot measurement variance to the probability distribution of the output state. Next, we analyze how depolarizing noise, a common form of quantum noise, alters this distribution. Finally, based on the noise-affected distribution, we establish a guaranteed lower bound on the single-shot measurement variance that holds uniformly over all inputs $x$.

Our analysis begins with a general quantum circuit that prepares the state $|\psi(x)\rangle$ from a given input $x$. Thus, the following theories and bounds for the general circuit apply directly to both positive- and negative-shifted circuits in Algorithm~\ref{alg:dp-prs}. Let $O$ be an observable with a spectral decomposition $O = \sum_i \lambda_i P_i$, where $\lambda_i$ are the distinct eigenvalues and $P_i$ are the projectors to the corresponding eigenspaces. The expectation of $O$ with respect to the output is
$
\langle O \rangle_x = \bra{\psi(x)} O \ket{\psi(x)}$. This expectation is practically estimated through repeated measurements (shots). 
A single shot yields one of the eigenvalues $\lambda_i$, with probability
\begin{equation}
    p_i(x) = \Pr(\hat{a} = \lambda_i) = \bra{\psi(x)} P_i \ket{\psi(x)}
\end{equation}
Performing $N_s$ independent shots on identically prepared states produces i.i.d. outcomes $\hat{a} \in \{\lambda_i\}$ drawn from the distribution $\{p_i(x)\}$. 

We denote the variance of a single-shot measurement by $\eta_{\mathrm{shot}}^2(x)$. 
It is given by:
\begin{equation}
    \begin{aligned}
    \eta_{\mathrm{shot}}^2(x) 
    &= \mathbb{E}[\hat{a}^2] - (\mathbb{E}[\hat{a}])^2 \\
    &= \sum_i \lambda_i^2 p_i(x) - \left( \sum_i \lambda_i p_i(x) \right)^2
\end{aligned}
\end{equation}

By definition, $\sigma^2_{\text{shot}}$ is a lower bound on the single-shot variance for both the positive- and negative-shifted circuits. 
This interpretation remains valid if we generalize the definition by assigning:
\begin{equation}
    \sigma^2_{\text{shot}} := \inf_{x} \, \eta_{\mathrm{shot}}^2(x)
\end{equation}

We observe that if there exists some $x$ and $i$ such that $p_i(x) = 1$, the single-shot variance $\sigma_{\text{shot}}$ becomes trivially zero. Consequently, without this inherent statistical noise, from Theorem~\ref{thm:main}, \NAMEA loses its primary advantage, as the required artificial noise $\sigma$ cannot be reduced. In practice, however, quantum noise in quantum circuits prevents this situation from occurring. 

We now consider the effect of depolarizing noise, which is a common noise in quantum circuits. Under this noise model, the output state becomes mixed, 
which leads to a non-trivial lower bound on $\sigma_{\text{shot}}$.

To represent quantum states in a noisy environment, we employ the density matrix formalism. 
The ideal state is given by $\rho_{\text{ideal}}(x) = \ket{\psi(x)}\bra{\psi(x)}$. 
Suppose this state passes through a global depolarizing channel $\mathcal{E}_{\alpha}$ of strength $\alpha \in [0,1]$ before measurement. 
The resulting (noisy) state is
\begin{equation}
    \rho_{\text{noisy}}(x) = \mathcal{E}_{\alpha}\big(\rho_{\text{ideal}}(x)\big) 
    = (1-\alpha)\rho_{\text{ideal}}(x) + \alpha \frac{I}{d},
\end{equation}
where $I$ is the identity operator and $d$ is the Hilbert space dimension. 
The state $\frac{I}{d}$ is known as \emph{uniform state}. 
When measured with respect to any observable, the outcomes are distributed uniformly, i.e., 
$p_i^{\text{uniform}} = \tfrac{1}{d}$ for all $i$. 
Thus, the single-shot variance under the uniform state is a constant:
\begin{equation}
    \sigma_{\mathrm{uniform}}^2 
    = \left(\frac{1}{d}\sum_i \lambda_i^2\right) 
      - \left(\frac{1}{d}\sum_i \lambda_i\right)^2
\end{equation}

\begin{theorem}
    The single-shot variance of a measurement of an observable $O$ on this noisy state, $\eta_{\mathrm{shot}}^2(x)$, is lower-bounded by:
    \begin{equation}
        \eta_{\mathrm{shot}}^2(x) \ge \alpha \cdot \sigma_{\mathrm{uniform}}^2
    \end{equation}
    \label{thm:shot_variance_lower_bound}
\end{theorem}

Theorem~\ref{thm:shot_variance_lower_bound} indicates that in the presence of noise, 
the single-shot variance for any input $x$ is lower-bounded by 
$\alpha \cdot \sigma_{\mathrm{uniform}}^2$. 
Intuitively, the depolarizing channel $\mathcal{E}_\alpha$ mixes the ideal state with the maximally mixed state, 
whose variance is maximized. 
As a result, the variance of the noisy state is always guaranteed to exceed a non-trivial threshold. 
Applying the noise channel $\mathcal{E}_\alpha$ to the outputs of the positive- and negative-shifted circuits in \NAMEA, 
we conclude that
$
    \sigma_{\mathrm{shot}}^2 \;\ge\; \alpha \cdot \sigma_{\mathrm{uniform}}^2,
$
thus ensuring a meaningful lower bound. 
This, in turn, reduces the amount of artificial noise $\sigma^2$ that needs to be injected (see Corollary~\ref{corollary:depo noise}).

\begin{corollary}
    The presence of physical depolarizing noise on the quantum circuit (i.e., $\alpha > 0$) directly reduces the amount of artificial noise $\sigma$ required by \NAMEA algorithm.
    \label{corollary:depo noise}
\end{corollary}

\begin{figure*}[t]
    \centering
    \includegraphics[width=0.8\linewidth]{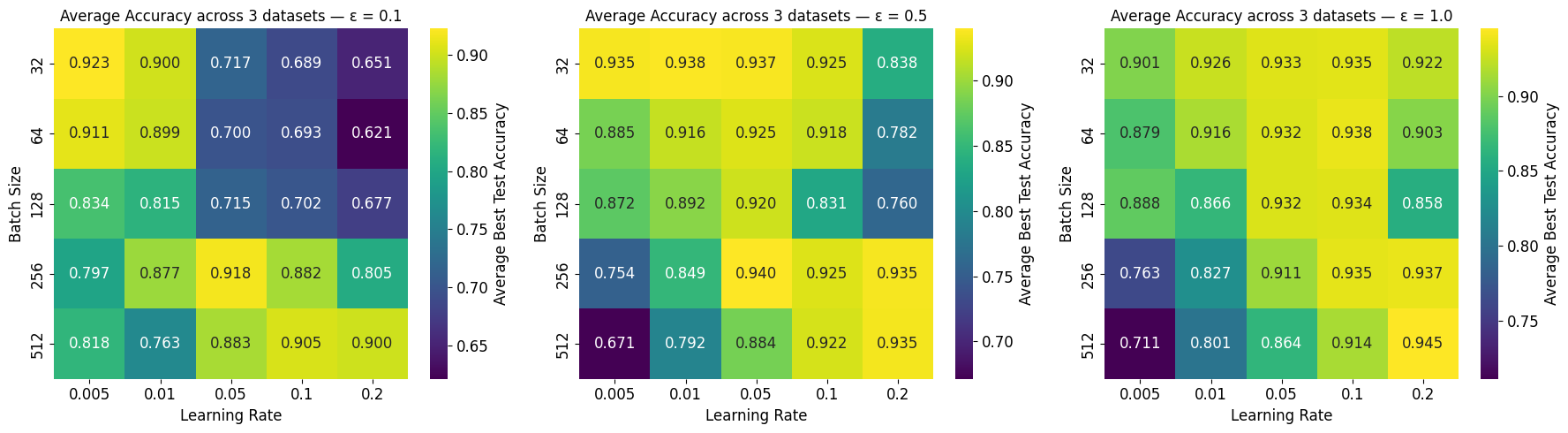}
    \caption{ Average test accuracy on 3 datasets using Q-ShiftDP w.r.t batch size, learning rates and $\varepsilon$.}
    \label{fig:average}
\end{figure*}

\subsection{Estimation of Single-Shot Measurement Variance by Samples}

The central guarantee of Theorem~\ref{thm:main} relies on the single-shot variance, $\sigma^2_{\mathrm{shot}}$, to reduce the required artificial noise. A global lower bound on this variance can be conservative. The actual variance for many samples may be much higher than this worst-case floor. %That presents a privacy advantage that a global bound cannot exploit. In this section, 
Thus, we propose an adaptive algorithm that replaces the conservative bound with an empirical estimate of shot-noise variance from measurement outcomes. This approach allows us to tailor the amount of added privacy noise based on the overall randomness of each sample's measurements. To make this approach practical and simplify the analysis, we assume all generator frequencies are identical, $\Omega_k = \Omega \forall k$. This is a common scenario in variational circuits built from repeating layers of the same learnable gates (e.g., $R_Y$).

%In Theorem~\ref{thm:main}, the artificial noise multiplier $\sigma^2$ needs to be large enough to guarantee DP, which in turn depends on the minimum per-sample shot noise $\inf_x \eta__{\mathrm{shot}}^2_{\mathrm{shot}}(x)$. In most case, the per-sample noise $\eta__{\mathrm{shot}}^2_{\mathrm{shot}}(x)$ is not zero \nhat{(need plot to show this)}, but its infimum cannot be lower-bounded without some additional assumptions, for example as in~\ref{sec:quantum_noise_error_bound}. In this section, instead of choosing a fixed global noise multiplier based on a global lower bound, we propose to add noise \emph{adaptively} on a per-sample basis. In particular, we compute an estimator $\hat{\eta__{\mathrm{shot}}}^2_{\mathrm{shot}}(x)$ for each sample using the observations. This allows us to take advantage of the inherent noise induced by the quantum measurement process to reduce the noise needed to maintain privacy, thus improving model utility.

%Consider the same setup from Section~\ref{sec:quantum_noise_error_bound}. 
%We first examine the estimation of the single-shot variance. As being discussed in Section~\ref{sec:quantum_noise_error_bound}, performing independent measurements on identically prepared states is equivalent to sampling from an unknown categorical distribution with the single-shot variance $\eta__{\mathrm{shot}}^2$. Before any DP-related noise is added, the sampling process itself induces statistical noise in the the observed value for $\langle O \rangle_x$ used for gradient computation. 

Given an input $x_j$ and weight $\theta_k$, let the true single-shot variances of the positive- and negative-shifted circuits be $\eta_{j,k,+}^2$ and $\eta_{j,k,-}^2$, respectively. We define the \emph{total batch shot-noise variance}, $\bm{\eta}_{B}^2$, as the sum of all single-shot variances from both the positive and negative-shifted circuits for every sample and trainable weight in a batch $B$, i.e., $\bm{\eta}_{B}^2 = \sum_j \sum_k \sum_{\tau \in \pm} \eta^2_{j,k,\tau}$. We can achieve the same privacy guarantee by assigning a batch-specific noise level $\sigma_B$, rather than adding a global noise $\sigma$ as in Algorithm~\ref{alg:dp-prs}, such that:
%Given the sum of true single-shot variances in a batch $B$ as $\bm{\eta}_{B}^2 = \sum_j \sum_{\tau \in \pm} \eta^2_{j,\tau}$, based on the proof of Lemma~\ref{lemma:iterationDP}, we can present the variance of sum of gradients in a batch $B$ exactly as $\operatorname{Var}[\mathbf{\tilde{g}}] = \frac{ \bm{\eta}_{B}^2}{4B^2N_s}\left(\sum_{k=1}^K \Omega_k^2\right) + \frac{\sigma^2 \Delta^2}{B^2}$. As a result,  the DP is guaranteed if for every batch $B$, we add an amount of artificial noise $\sigma_B$ satisfying:
\begin{equation}
    \sigma^{2}_B \ge  \max\left(0, C_{\text{DP}} - \frac{\Omega^2\bm{\eta}_{B}^2}{4N_s \Delta^2} \right)
\end{equation}
%Since the estimation procedures for the two cases are analogous, we use the unified notation $\eta_{\pm}^2$ to present the method. 
However, $\bm{\eta}_B^2$ is unknown a priori. Based on the outcomes of the positive- and negative-shifted circuits, we aim to construct an estimator $\hat{\bm{\eta}}_B^2$ for the total true variance $\bm{\eta}_B^2$. This estimator is then used to determine the batch-specific artificial noise $\sigma_B^2$.
%\nguyendo{I should denote 2 complex formula terms as 2 auxiliary variables to avoid margin violation}
%The purpose of this variance estimator is to preserve differential privacy while minimizing unnecessary noise injection. 
Ideally, the privacy guarantee level remains intact if the estimator is conservative, i.e., $\bm{\eta}_B^2 \ge \hat{\bm{\eta}}_B^2$. %At the same time,  
We also want the estimator to be as close as possible to the true variance, thereby reducing %in order to reduce 
the amount of artificial noise added.

We now describe our variance estimation technique, which is based on Lemma~\ref{lemma:pivot_quantitity}. 
This lemma guarantees, with high probability, that $\bm{\eta}_B^2$ is lower-bounded by a function of sample variances and sample fourth central moments, both of which are computable from measurement outcomes. 
Given sets of outcomes for each sample $j$ and weight $k$ in the batch, we denote the computable sample variances as $\bar{\eta}^2_{j,k,\pm}$ and the sample fourth central moments as $\bar{\mu}_{4,j,k,\pm}$. 
Our variance estimator is then formulated as:
\begin{equation}
    \hat{\bm{\eta}}_B^2 = \bar{\bm{\eta}}_B^2 - z_\beta\sqrt{\sum_{j=1}^B \sum_{k=1}^K \sum_{\tau\in \pm} \frac{1}{N_s}\left(\bar{\mu}_{4,j,k,\tau} - (\bar{\eta}_{j,k,\tau}^2)^2\right)}
    \label{eq:estimation_variance}
\end{equation}
where $\bar{\bm{\eta}}_B^2 = \sum_{j=1}^B \sum_{k=1}^K\sum_{\tau\in \{+,-\}} \bar{\eta}^2_{j,k,\tau}$.

\begin{comment}
\begin{lemma}
\label{lem:pivot_quantitity}
Given the set of observations $\{\hat{a}_n\}_{i=1}^{N_s}$ (or $\{\hat{b}_n\}_{i=1}^{N_s}$), let $\bar{\eta}_{\pm}^2$ and $\bar{\mu}^{(4)}_{\pm}$ be the sample variance and sample fourth central moment from these observations, respectively, and let $z_\beta$ be the critical value for the normal distribution at significance level $\beta$. Assuming that the true fourth central moment $\bar{\mu}^{(4)}_{\pm}$ is finite, then:
$$
\operatorname{Pr}\left(\eta_{\pm}^2 \geq \bar{\eta}_{\pm}^2 - z_\beta\sqrt{\frac{1}{N_s}\left(\bar{\mu}^{(4)}_{\pm} - (\bar{\eta}_{\pm}^2)^2\right)}\right) \to 1 - \beta
$$
as $N_s \to \infty$.
\end{lemma}
\end{comment}

\begin{comment}
\begin{lemma}
\label{lem:pivot_quantitity}
    Let $c_1, \dots, c_{n}$ be independent and identically distributed (i.i.d.) samples drawn from a distribution with true variance $\eta^2$ and a finite true fourth central moment $\mu_4$. Let $\bar{\eta}^2_{n}$ and $\bar{\mu}_{4, n}$ be the sample variance and sample fourth central moment, respectively, computed from these observations. For a chosen significance level $\beta$, let $z_\beta$ be the corresponding critical value of the standard normal distribution.

Then, the following probabilistic lower bound holds:
\begin{equation}
    \Pr\left(\eta^2 \geq \bar{\eta}^2_{n} - z_\beta\sqrt{\frac{1}{n}\left(\bar{\mu}_{4, n} - (\bar{\eta}^2_{n})^2\right)}\right) \to 1 - \beta
\end{equation}
as the number of samples $n \to \infty$.
\end{lemma}
\end{comment}

\begin{lemma}
\label{lemma:pivot_quantitity}
Let there be $L$ independent groups, indexed by $i=1, \dots, L$. For each group $i$, let $y_{i,1}, \dots, y_{i,n} \sim \text{i.i.d. } Y_i$ be $n$ observations. Let $\eta_i^2$, $\bar{\eta}_i^2$, $\mu_{4,i}$ and $\bar{\mu}_{4,i}$ be the true variance, sample variance, true fourth central moment, and sample fourth central moment for the $i$-th group. Let $\eta^2_{\textrm{batch}} = \sum_{i=1}^L \eta^2_i$ and $\bar{\eta}_{\textrm{batch}}^2 = \sum_{i=1}^L \bar{\eta}_i^2$. Assuming all $\mu_{4,i}$ are finite and given a significance level $\beta$ with critical value $z_\beta$, then:
$$
P\left(\eta^2_{\textrm{batch}} \geq \bar{\eta}^2_{\textrm{batch}} - z_\beta\sqrt{\sum_{i=1}^L \frac{1}{n}\left(\bar{\mu}_{4,i} - (\bar{\eta}_i^2)^2\right)}\right) \to 1 - \beta
$$
as $L \to \infty, n \to \infty$.
\end{lemma}

By using this probabilistic estimator, we propose an adaptive algorithm, named \NAMEB (see Algorithm~\ref{alg:adaptive-dp-prs} in Appendix~\ref{sec:add_sampling}), which tailors the amount of artificial noise for each mini-batch. The main idea is to calibrate the artificial noise per batch, $\sigma^2_B$, using a probabilistic estimator $\hat{\bm{\eta}}_B^2$. From Lemma~\ref{lemma:pivot_quantitity}, this estimator is conservative with high probability $\Pr(\bm{\eta}_B^2 \ge \hat{\bm{\eta}}_B^2) \to 1-\beta$. This reliance on a statistical estimate means the privacy guarantee itself is probabilistic. We can thus analyze the algorithm's behavior in two distinct \emph{Modes} for each batch. The mechanism operates in the \emph{Good Mode} when the estimate is a valid lower bound, ensuring the privacy guarantee holds (i.e., $\bm{\eta}_B^2 \geq \hat{\bm{\eta}}_B^2$). %\mt{These bad/good modes is quite confusing. What is bad, what is good Also do we have an adaptive alg in appendix?} 
%The assumed noise is less than or equal to the true noise, and the intended privacy guarantee holds. 
Conversely, in the \emph{Bad Mode}, the estimator is an overestimate, thereby compromising the privacy guarantee (i.e., $\bm{\eta}_B^2 < \hat{\bm{\eta}}_B^2$). Lemma~\ref{lemma:mechanism_probabilistic_dp} provides a formal tool to analyze the overall privacy of such a probabilistic mechanism by converting the probability of success into a standard DP guarantee.

\begin{lemma}
\label{lemma:mechanism_probabilistic_dp}
Let $\mathcal{M}$ be a randomized mechanism that operates in one of two modes. With probability $\gamma$, it operates in a ``Good Mode,'' where its execution is equivalent to a mechanism $\mathcal{M}_{good}$ that is $(\varepsilon, \delta)$-differentially private. With probability $1-\gamma$, it operates in a ``Bad Mode,'' where no privacy guarantee holds. Then, the overall mechanism $\mathcal{M}$ is $(\varepsilon, \gamma\delta + (1-\gamma))$-DP.
\end{lemma}

From Lemma~\ref{lemma:pivot_quantitity}, the probability of \emph{Good Mode} is $1-\beta$, while that of \emph{Bad Mode} is $\beta$. As discussed above, in \emph{Good Mode}, the privacy level remains $(\varepsilon, \delta)$. Consequently, our variance estimation technique attains the DP guarantee formalized in Theorem~\ref{thm:advanced_adaptive_privacy}. More details of \NAMEB are included in Appendix~\ref{sec:add_sampling}.

%$$
%\hat{\eta}^2_{+}(x) = \frac{1}{N_s - 1} \sum_{i=1}^{N_s} (\hat{a}_n - \bar{a})^2
%$$
%where $\bar{a} = \frac{1}{N_s}\sum_{i=1}^{N_s} \hat{a}_n$ is the sample mean of the outcomes. Similarly, the estimator, $\hat{\eta}^2_{-}(x)$, can be constructed from the outcomes $\{\hat{b}_i\}_{i=1}^{N_s}$ of the negative-shifted circuit.

%thus obtaining a probabilistic lower bound of the sample noise for a desired significance level. The following lemma formalizes this:
\begin{theorem}
\label{thm:advanced_adaptive_privacy}
For a mini-batch of size $B$, let $\hat{\bm{\eta}}_B^2$ be a probabilistic estimator for the true sum of all variances, $\bm{\eta}_B^2 = \sum_{j=1}^{B} \sum_{k=1}^K \sum_{\tau \in \pm} \eta_{j,k,\tau}^2$ such that it satisfies $\Pr(\bm{\eta}_B^2 \ge \hat{\bm{\eta}}_B^2) \to 1-\beta$ for a chosen significance level $\beta$.
\NAMEB calculates a noise multiplier for each batch $B$, $\sigma^2_B$, satisfying:
$$
\sigma^2_B \ge \max\left(0, C_{\text{DP}} - \frac{\Omega^2\hat{\bm{\eta}}_{B}^2}{4N_s \Delta^2} \right)
$$
where $C_{\text{DP}} = \ \left( c_2 \frac{q \sqrt{T \log(1/\delta)}}{\varepsilon} \right)^2$ with $c_2$ is a constant. \NAMEB achieves \textbf{$(\varepsilon, (1-\beta)\delta + \beta)$-DP}.
\end{theorem}

\begin{figure*}[t]
  \centering
  \begin{subfigure}{0.33\linewidth}
    \centering
    \includegraphics[width=\linewidth]{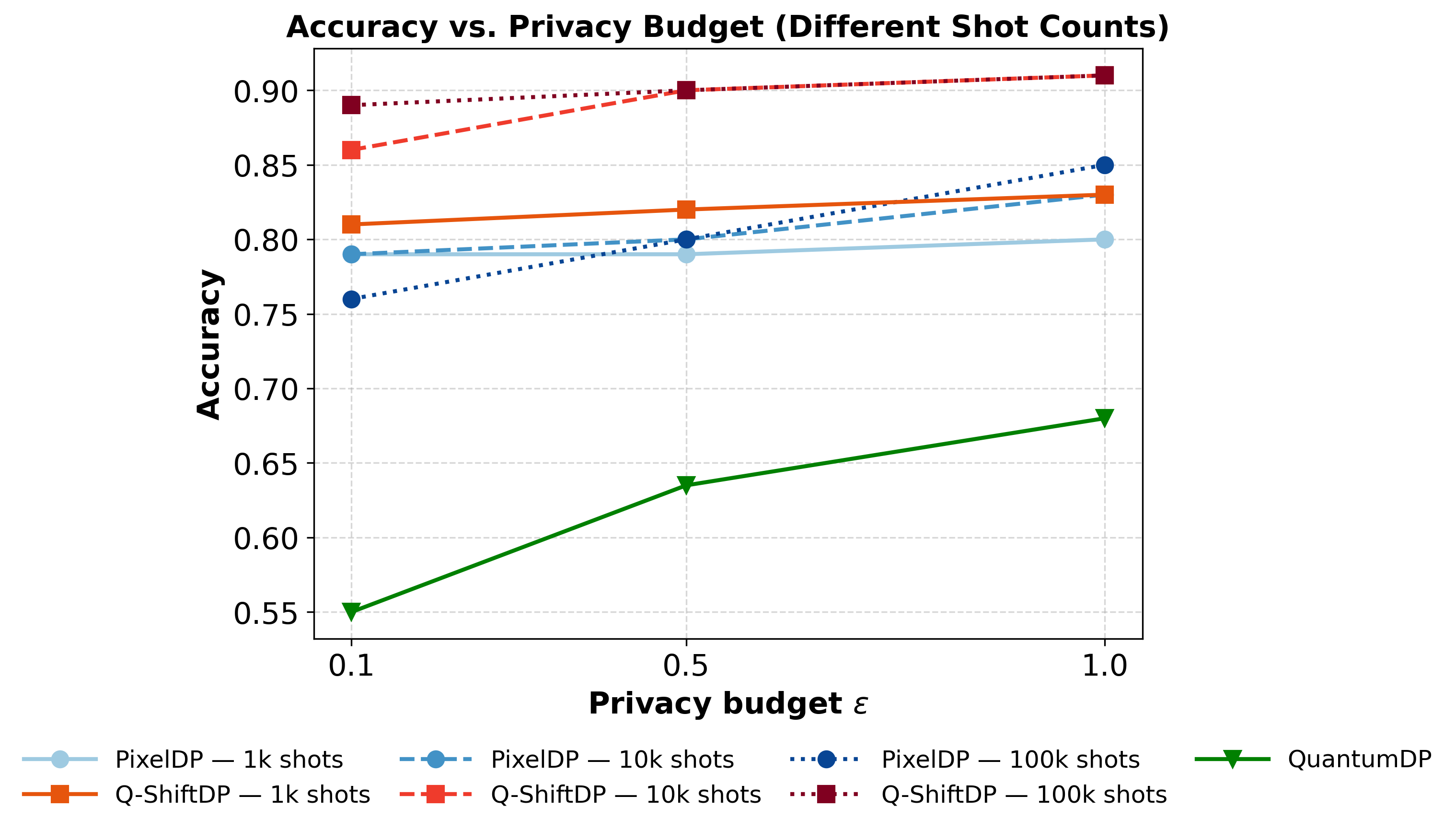}
    \caption{Bars \& Stripes}
    \label{fig:compare:qshift}
  \end{subfigure}\hfill
  \begin{subfigure}{0.33\linewidth}
    \centering
    \includegraphics[width=\linewidth]{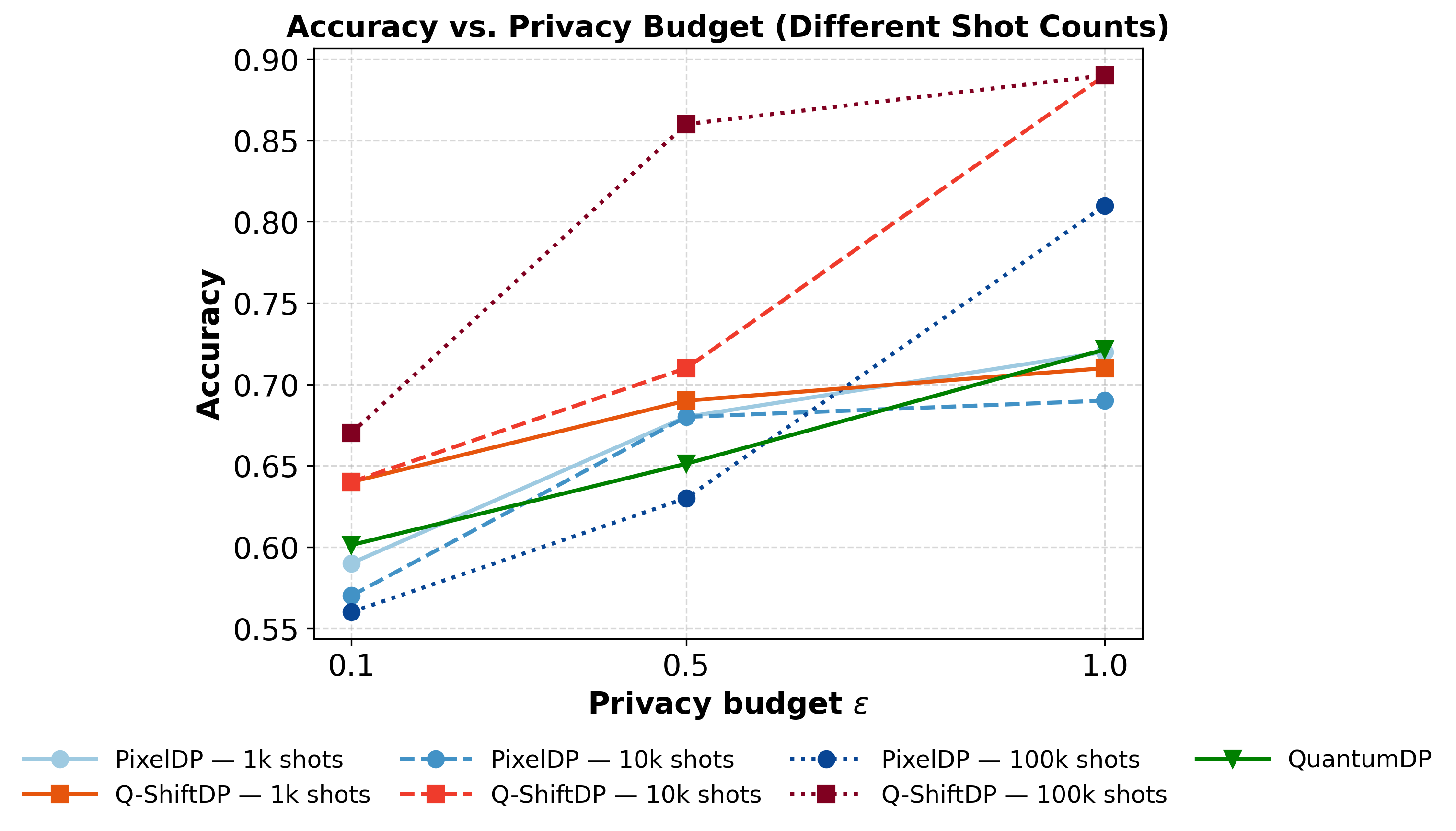}
    \caption{Downscaled MNIST}
    \label{fig:compare:qshift}
  \end{subfigure}\hfill
  \begin{subfigure}{0.33\linewidth}
    \centering
    \includegraphics[width=\linewidth]{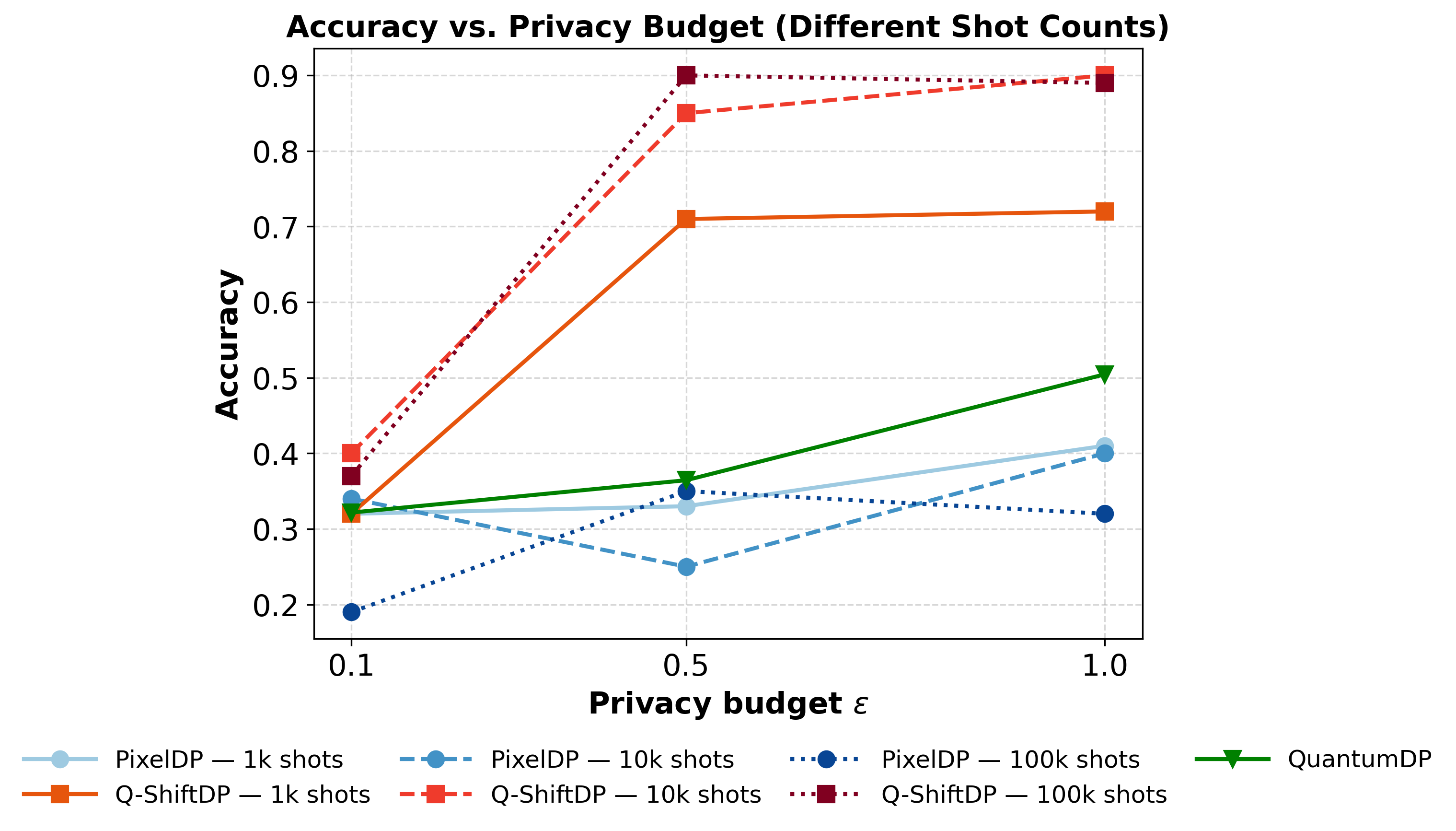}
    \caption{Binary Blobs}
    \label{fig:compare:pixel}
  \end{subfigure}
  \caption{Comparison of model utility under equal privacy budgets: Q-ShiftDP vs. Pixel-level DP vs. QuantumDP in various datasets.}
  \label{fig:compareall}
\end{figure*}

% --- Second figure ---
\begin{figure}[t]
    \centering
    \includegraphics[width=0.60\linewidth]{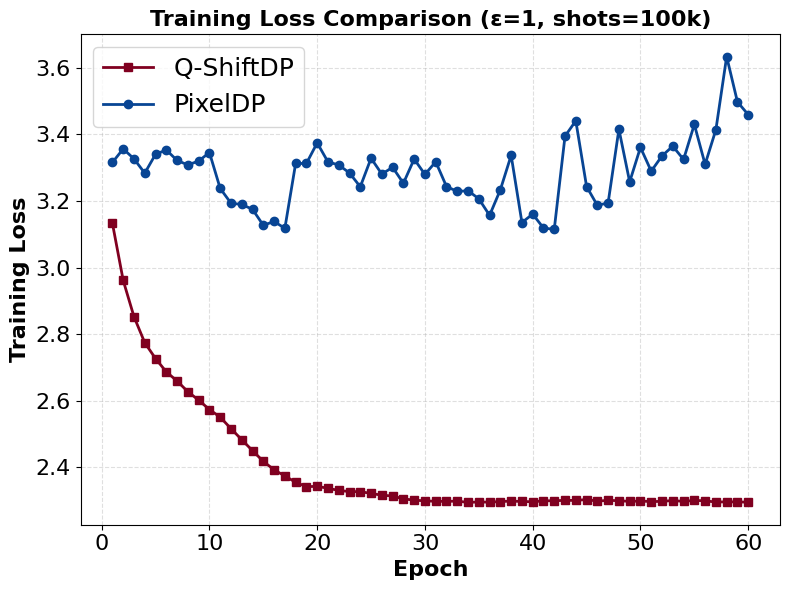}
    \caption{Training loss curves for Q-ShiftDP and Pixel-level DP under $\varepsilon=1$ with 100k shots.}
    \label{fig:trainloss}
\end{figure}

\section{Experimental Results}
We evaluate the performance and privacy–utility trade-off of Q-ShiftDP on three representative datasets from recent QML benchmark suites:  Bars \& Stripes, Binary Blobs and Downscaled MNIST \cite{bowles2024better, recio2025train}. These benchmarks are carefully constructed, low-dimensional tasks that align with quantum circuit constraints and allow meaningful, reproducible evaluation of quantum models. Our classifier is a 4-qubit quantum neural network with a single Strongly Entangling Layer for Bars and Stripes dataset and 5 layers for other datasets. Details about datasets and models are given in Appendix \ref{appx:dataset_model}. 
% Our classifier is a 4-qubit quantum neural network. For Bars and Stripes, we use a single Strongly Entangling Layer. For Binary Blobs and Downscaled MNIST, we use a 5-layer Strongly Entangling Layers ansatz. 

\textbf{Performance of Q-ShiftDP under Varying Hyperparameter Settings.}  
The performance of differentially private optimizers is highly sensitive to hyperparameter choices \cite{bu2023automatic}. Unlike DP-SGD, Q-ShiftDP leverages the parameter-shift rule and the bounded spectrum of quantum observables to derive an exact analytic sensitivity bound, removing the need to tune the clipping norm. Figure~\ref{fig:average} reports the average test accuracy across all datasets for different combinations of learning rate, batch size, and privacy budget, with detailed results provided in Appendix~\ref{appx:hyperparam_tuning}. We observe that large batch sizes paired with moderately high learning rates consistently yield the most stable performance. This result is consistent with previous DP-SGD literature \cite{de2022unlocking}. Based on these findings, all subsequent experiments use a batch size of $B=512$ and a learning rate of $\text{lr}=0.2$. In this experiment, we use analytic expectation values rather than limiting the number of measurement shots. The effect of finite shots on Q-ShiftDP’s performance is further analyzed in Appendix~\ref{appx:shots}.

\textbf{Comparison to Existing Privacy Baselines} 
We compare Q-ShiftDP against PixelDP~\cite{lecuyer2019certified} and QuantumDP~\cite{Hirche2023InfoTheoryQDP}. While PixelDP enforces differential privacy by injecting classical noise directly into the input data, QuantumDP achieves DP by relying solely on intrinsic quantum noise (i.e., depolarizing channels). Importantly, we view Q-ShiftDP as a theoretically optimized adaptation of DP-SGD for QML. Thus, we do not include a direct empirical comparison to DP-SGD, as Q-ShiftDP theoretically dominates DP-SGD when applied to QML.

Figure~\ref{fig:compareall} reports test accuracy as a function of the privacy budget $\varepsilon$ for different shot counts across three benchmark datasets. For QuantumDP, results are reported using $100\text{k}$ shots. Across all datasets and experimental settings, Q-ShiftDP consistently outperforms both PixelDP and QuantumDP under the same privacy budget and shot count. The performance gap is particularly pronounced in the strong privacy regime ($\varepsilon = 0.1$), where input perturbation and quantum-only noise lead to substantial utility degradation. Figure~\ref{fig:trainloss} further compares the training dynamics of Q-ShiftDP and PixelDP. Under $\varepsilon = 1$ with $100\text{k}$ shots, Q-ShiftDP exhibits significantly more stable training behavior than PixelDP. Taken together, these results suggest that jointly accounting for both noise sources, as done in Q-ShiftDP, yields a more favorable privacy–utility trade-off in private quantum learning.

\textbf{Additional Experiments.} We provide further experimental results and analyses in the Appendix~\ref{appx:extra_exp}.

\section{Conclusion}

In this work, we introduce \NAMEA, a privacy mechanism designed for quantum machine learning. By proving that quantum gradients are inherently bounded and their intrinsic shot noise is a quantifiable privacy resource, our method eliminates the need for manual clipping and provably reduces the amount of artificial noise required. We established formal privacy and utility guarantee, further enhanced by an adaptive technique that tailors the noise to the specifics of each batch. Experimental results show that \NAMEA outperforms the existing benchmark in multiple settings.

\newpage

\bibliography{bibliography}
\bibliographystyle{plainnat}

%%%%%%%%%%%%%%%%%%%%%%%%%%%%%%%%%%%%%%%%%%%%%%%%%%%%%%%%%%%%
\section*{Checklist}

\begin{enumerate}

  \item For all models and algorithms presented, check if you include:
  \begin{enumerate}
    \item A clear description of the mathematical setting, assumptions, algorithm, and/or model. [Yes]
    \item An analysis of the properties and complexity (time, space, sample size) of any algorithm. [Yes]
    \item (Optional) Anonymized source code, with specification of all dependencies, including external libraries. [Yes, link to anonymized source code is included in the appendix.]
  \end{enumerate}

  \item For any theoretical claim, check if you include:
  \begin{enumerate}
    \item Statements of the full set of assumptions of all theoretical results. [Yes]
    \item Complete proofs of all theoretical results. [Yes]
    \item Clear explanations of any assumptions. [Yes]     
  \end{enumerate}

  \item For all figures and tables that present empirical results, check if you include:
  \begin{enumerate}
    \item The code, data, and instructions needed to reproduce the main experimental results (either in the supplemental material or as a URL). [Yes]
    \item All the training details (e.g., data splits, hyperparameters, how they were chosen). [Yes]
    \item A clear definition of the specific measure or statistics and error bars (e.g., with respect to the random seed after running experiments multiple times). [Yes, except error statistics are omitted due to the high computational cost of running repeated quantum simulations.]
    \item A description of the computing infrastructure used. (e.g., type of GPUs, internal cluster, or cloud provider). [Yes]
  \end{enumerate}

  \item If you are using existing assets (e.g., code, data, models) or curating/releasing new assets, check if you include:
  \begin{enumerate}
    \item Citations of the creator If your work uses existing assets. [Yes]
    \item The license information of the assets, if applicable. [Not Applicable]
    \item New assets either in the supplemental material or as a URL, if applicable. [Not Applicable]
    \item Information about consent from data providers/curators. [Not Applicable]
    \item Discussion of sensible content if applicable, e.g., personally identifiable information or offensive content. [Not Applicable]
  \end{enumerate}

  \item If you used crowdsourcing or conducted research with human subjects, check if you include:
  \begin{enumerate}
    \item The full text of instructions given to participants and screenshots. [Not Applicable]
    \item Descriptions of potential participant risks, with links to Institutional Review Board (IRB) approvals if applicable. [Not Applicable]
    \item The estimated hourly wage paid to participants and the total amount spent on participant compensation. [Not Applicable]
  \end{enumerate}

\end{enumerate}

\clearpage
\appendix
\thispagestyle{empty}

% Supplementary material: To improve readability, you must use a single-column format for the supplementary material.
\onecolumn
\aistatstitle{Supplementary Materials}% Reset theorem-like environments
\setcounter{lemma}{0}
\setcounter{theorem}{0}

\section{Additional Backgrounds}
\label{sec:add_background}
 %On the other hand, the evolution of quantum states is governed by quantum operations, which describe how a state changes under physical processes. In ideal closed systems, these operators can be described as unitary matrices $U$, where the state transformation is given by $\ket{\psi'} = U\ket{\psi}$. 

\noindent \textbf{Quantum Systems in Noisy Environment.} Realistic quantum systems are inherently open, interacting with their surrounding environment. Thus, it is subject to noise, decoherence, and operational imperfections. In this context, the description of quantum states is generalized from pure states to density matrices. A density matrix can represent mixed states, where the system is described as a statistical ensemble of pure states $\{p_i, \ket{\psi_i}\}$ as follow:
\begin{equation}
    \rho = p_i \ket{\psi_i}\bra{\psi_i}, \quad \sum_i p_i = 1
\end{equation}

The dynamics of such states are modeled by quantum channels, which are completely positive, trace-preserving (CPTP) maps acting on $\rho$. A general channel $\mathcal{E}$ follows a Kraus representation,
\begin{equation}
    \mathcal{E}(\rho) = \sum_i K_i \rho K_i^\dagger, 
\quad \sum_i K_i^\dagger K_i = I
\end{equation}
where $\{K_i\}$ are Kraus operators. Depolarizing channel is an example of quantum noisy channel. Specifically, in $d$-dimensional system, it replaces the input state with the maximally mixed state $I/d$ with probability $\alpha$, and leaves it unchanged with probability $1-\alpha$, i.e., 
\begin{equation}
    \mathcal{E}_{\alpha}(\rho) = (1-\alpha)\rho + \alpha \frac{I}{d}
\end{equation}

\section{Proofs}
\label{sec:proofs}

\begin{lemma}
    Given the objective function $C(\boldsymbol{\theta}) = \langle\psi_0|\hat{U}(\boldsymbol{\theta})^{\dagger}\hat{O}\hat{U}(\boldsymbol{\theta})|\psi_0\rangle$ where $\hat{O}$ is an observable with a minimum and maximum eigenvalues $\lambda_{\text{min}}$ and $\lambda_{\text{max}}$ respectively and $\hat{U}(\boldsymbol{\theta})$ is a product of unitary operators. The value of $C(\boldsymbol{\theta})$ is bounded by these same extremal eigenvalues for all possible parameter values $\boldsymbol{\theta}$.
$$
\lambda_{\min} \le C(\boldsymbol{\theta}) \le \lambda_{\max}
$$
\label{lemma:appx:eigenvalues}
\end{lemma}

\begin{proof}
    In a finite-dimensional space, $\hat{O}$ can be expressed in terms of its eigenvalues $\lambda_i$ and corresponding orthonormal eigenvectors $\ket{e_i}$ as:
    $$\hat{O} = \sum_{i} \lambda_i \ket{e_i}\bra{e_i}$$

    where $\hat{O}\ket{e_i} = \lambda_i \ket{e_i}$ and $\braket{e_i|e_j} = \delta_{ij}$. The set of eigenvectors $\{\ket{e_i}\}$ forms an orthonormal basis in the Hilbert space $\mathcal{H}$.

    We call the state of the system after the unitary evolution as $\ket{\psi(\boldsymbol{\theta})} = \hat{U}(\boldsymbol{\theta})\ket{\psi_0}$. Then, the objective function can be rewritten as:
    $$
    C(\boldsymbol{\theta}) = \braket{\psi(\boldsymbol{\theta}) | \hat{O} | \psi(\boldsymbol{\theta})}
    $$

    On the other hand, we can express $\ket{\psi(\boldsymbol{\theta})}$ in the eigenbasis of $\hat{O}$:
    $$
    \ket{\psi(\boldsymbol{\theta})} = \sum_i c_i \ket{e_i}, \quad \text{where } c_i = \braket{e_i|\psi(\boldsymbol{\theta})}
    $$

    By substitution, we have:
    \begin{align*}
C(\boldsymbol{\theta}) &= \braket{\psi(\boldsymbol{\theta}) | \left( \sum_j \lambda_j \ket{e_j}\bra{e_j} \right) | \psi(\boldsymbol{\theta})} \\
&= \sum_j \lambda_j \Braket{\psi(\boldsymbol{\theta}) | e_j}\Braket{e_j | \psi(\boldsymbol{\theta})} \\
&= \sum_j \lambda_j |\braket{e_j | \psi(\boldsymbol{\theta})}|^2 \\
&= \sum_j |c_j|^2 \lambda_j
\end{align*}

This shows that $C(\boldsymbol{\theta})$ is a convex combination of the eigenvalues of $\hat{O}$ because $|c_j|^2 \ge 0$ and $\sum_j |c_j|^2 = 1$ (the normalization condition of quantum state $\ket{\psi(\boldsymbol{\theta})}$). As a result, we have the upperbound and lowerbound for $C(\boldsymbol{\theta})$ as following.

For the lower bound:
$$
C(\boldsymbol{\theta}) = \sum_j |c_j|^2 \lambda_j \ge \sum_j |c_j|^2 \lambda_{\min} = \lambda_{\min} \left( \sum_j |c_j|^2 \right) = \lambda_{\min}
$$

For the upperbound:
$$
C(\boldsymbol{\theta}) = \sum_j |c_j|^2 \lambda_j \le \sum_j |c_j|^2 \lambda_{\max} = \lambda_{\max} \left( \sum_j |c_j|^2 \right) = \lambda_{\max}
$$

\end{proof}

\begin{comment}
\begin{lemma}[L2-Sensitivity of the Quantum Gradient]
    For each sample $(x_j,y_j)$, we define the objective function $C(\boldsymbol{\theta}) = \Braket{\psi_0(x_j) | \hat{U}(\boldsymbol{\theta})^{\dagger} \hat{O}_{y_j} \hat{U}(\boldsymbol{\theta}) | \psi_0(x_j)}$ where the observable $\hat{O}_{y_j}$ has eigenvalues in range from $\lambda_{\min}$ to $\lambda_{\max}$. Let the gradient vector $\mathbf{g}^{(j)}$ be computed as in Algorithm~\ref{alg:dp-prs}, where each component is given by $$\mathbf{g}_k = \frac{\Omega_k}{2} (r_+ - r_-)$$ The \textbf{L2 norm} of this gradient vector is strictly bounded by the sensitivity $\Delta$:
    $$\|\mathbf{g}\|_2 \le \Delta \quad \text{where} \quad \Delta = \frac{\lambda_{\max} - \lambda_{\min}}{2} \sqrt{\sum_{k=1}^{K} \Omega_k^2}$$
    \label{lemma:appx:L2_sensitivity}
\end{lemma}
\end{comment}

\begin{lemma}[$\ell_2$-Sensitivity of the Quantum Gradient]
    We define an objective function $C(\boldsymbol{\theta}) = \Braket{\psi_0(x) | \hat{U}(\boldsymbol{\theta})^{\dagger} \hat{O} \hat{U}(\boldsymbol{\theta}) | \psi_0(x)}$ where the observable $\hat{O}$ has eigenvalues in range from $\lambda_{\min}$ to $\lambda_{\max}$ and parametric unitary $\hat{U}(\boldsymbol{\theta}) = \prod_{k=1}^{K} e^{i\theta_k G_k}$ with frequencies $\Omega_k$. 
Suppose the gradient of $C(\boldsymbol{\theta})$ with respect to parameter $\theta_k$ is computed using the parameter-shift rule:
\begin{equation}
        \mathbf{g}_k = \tfrac{\Omega_k}{2}\,(r_+ - r_-),
\end{equation}
%\nguyendo{should be inline equation}
where $r_+$ and $r_-$ are expectation values of the observable under a shifted amount $s = \frac{\pi}{2\Omega_k}$.  
    The \textbf{$\ell_2$-norm} of this gradient is strictly bounded by the sensitivity $\Delta$:
    \begin{equation*}
        \|\mathbf{g}_k\|_2 \le \Delta \quad \text{where} \quad \Delta = \frac{\lambda_{\max} - \lambda_{\min}}{2} \sqrt{\sum_{k=1}^{K} \Omega_k^2}
    \end{equation*}
    \label{lemma:appx:L2_sensitivity}
\end{lemma}

\begin{proof}
    First, we establish a bound on each individual component of the gradient and then combine them to bound the L2 norm. The $k$-th component of the gradient is $\mathbf{g}^{(j)}_k = \frac{\Omega_k}{2} (r_+ - r_-)$. From Lemma~\ref{lemma:appx:eigenvalues}, we have:
    $$
    \lambda_{\min} \le r_+ \le \lambda_{\max} \quad \text{and} \quad \lambda_{\min} \le r_- \le \lambda_{\max}
    $$

    As a result, we have:
    $$
    |\mathbf{g}^{(j)}_k| = \left|\frac{\Omega_k}{2} (r_+ - r_-)\right| \le \frac{|\Omega_k|}{2} (\lambda_{\max} - \lambda_{\min})
    $$

    The squared L2 norm is the sum of the squares of the components. Thus, we have:
    \begin{align*}
    \|\mathbf{g}^{(j)}\|_2^2 &= \sum_{k=1}^{d} \left( \mathbf{g}^{(j)}_k \right)^2 \\
    &\le \sum_{k=1}^{d} \left( \frac{|\Omega_k|}{2} (\lambda_{\max} - \lambda_{\min}) \right)^2 \\
    &= \left( \frac{\lambda_{\max} - \lambda_{\min}}{2} \right)^2 \sum_{k=1}^{d} \Omega_k^2
    \end{align*}

    Finally, we can conclude:
    $$
    \|\mathbf{g}^{(j)}\|_2 \le \frac{\lambda_{\max} - \lambda_{\min}}{2} \sqrt{\sum_{k=1}^{d} \Omega_k^2}
    $$
\end{proof}

\begin{comment}
\begin{lemma}
Assume that the finite-shot estimators $\hat{r}_+$ and $\hat{r}_-$ are independent random variables drawn from Gaussian distributions with the variance at least $\sigma_{\mathrm{shot}}^2/N_s$. Then, Algorithm~\ref{alg:dp-prs} ensures $(\varepsilon, \delta)-\operatorname{DP}$ per iteration with:
$$\varepsilon = \frac{\sqrt{2 \ln \frac{1.25}{\delta}}}{\sqrt{\frac{ 2B\sigma_{\mathrm{shot}}^2}{N_s(\lambda_{\max} - \lambda_{\min})^2} + \sigma^2}} $$
\label{lemma:appx:iterationDP}
\end{lemma}
\end{comment}

\begin{lemma}
Suppose $N_s$ is large enough so that the finite-shot estimators $\hat{r}_+$ and $\hat{r}_-$ are independent random variables drawn from Gaussian distributions with the variance at least $\sigma_{\mathrm{shot}}^2/N_s$. Then, Algorithm~\ref{alg:dp-prs} ensures $(\varepsilon_0, \delta_0)-\operatorname{DP}$ per iteration with:
\begin{equation}
    \varepsilon_0 = \frac{\sqrt{2 \ln \frac{1.25}{\delta_0}}}{\sqrt{\frac{ 2B\sigma_{\mathrm{shot}}^2}{N_s(\lambda_{\max} - \lambda_{\min})^2} + \sigma^2}}
\end{equation}
\label{lemma:appx:iterationDP}
\end{lemma}

\begin{proof}
    We consider the batch-averaged gradient:
    $$\mathbf{\tilde{g}} = \frac{1}{B} \sum_{j=1}^B \mathbf{g}^{(j)} + \frac{1}{B} \mathbf{z}$$
    where $\mathbf{g}^{(j)}$ is the per-example gradient and $z \sim \mathcal{N}(0, \sigma^2 \Delta^2 \mathbf{I}_K)$ is the Gaussian noise added for DP.

    Because $\sum_{j=1}^B \mathbf{g}^{(j)}$ and $\mathbf{z}$ are independent, we have $\mathbf{\tilde{g}}$ is Gaussian with:
    $$\operatorname{Var}[\mathbf{\tilde{g}}] = \operatorname{Var}[\frac{1}{B} \sum_{j=1}^B \mathbf{g}^{(j)}] + \operatorname{Var}[\frac{1}{B} \mathbf{z}]$$.

    The estimated $k$-th component of the gradient is $\mathbf{g}_k^{(j)} = \frac{\Omega_k}{2}(r_+ - r_-)$. Its variance is:
    \begin{align*}
    \operatorname{Var}[\mathbf{g}_k^{(j)}] &= \operatorname{Var}\left[\frac{\Omega_k}{2}(r_+ - r_-)\right] = \left(\frac{\Omega_k}{2}\right)^2 \operatorname{Var}[r_+ - r_-] \\
    &= \frac{\Omega_k^2}{4} \left(\operatorname{Var}[r_+] + \operatorname{Var}[r_-]\right) \ge \frac{\Omega_k^2}{4} (2\sigma_{\mathrm{shot}}^2/N_s) = \frac{\Omega_k^2 \sigma_{\mathrm{shot}}^2}{2N_s}
    \end{align*}

    Then, the total variance of the statistical noise vector is the sum of the component variances:
    $$
    \operatorname{Var}[\mathbf{g}^{(j)}] = \sum_{k=1}^K \operatorname{Var}[\mathbf{g}^{(j)}_k] \ge \frac{\sigma_{\mathrm{shot}}^2}{2N_s} \sum_{k=1}^K \Omega_k^2
    $$

On the other hand, the variance of the explicitly added Gaussian noise is:
    $$\operatorname{Var}[\frac{1}{B} \mathbf{z}] = \frac{\sigma^2 \Delta^2}{B^2}$$
where $\Delta = \frac{\lambda_{\max} - \lambda_{\min}}{2} \sqrt{\sum_{k=1}^{K} \Omega_k^2}$

Applying the Gaussian mechanism with sensitivity $\Delta/B$ then gives per-iteration $(\varepsilon,\delta)$-DP~\cite{Dwork2006DifferentialP}:
$$
\varepsilon = \frac{\Delta/B}{\sqrt{\frac{ \sigma_{\mathrm{shot}}^2}{2BN_s}\left(\sum_{k=1}^K \Omega_k^2\right) + \frac{\sigma^2 \Delta^2}{B^2}}} \sqrt{2 \ln \frac{1.25}{\delta}} = \frac{\sqrt{2 \ln \frac{1.25}{\delta}}}{\sqrt{\frac{ 2B\sigma_{\mathrm{shot}}^2}{N_s(\lambda_{\max} - \lambda_{\min})^2} + \sigma^2}}
$$

\end{proof}

\begin{theorem}
    Consider Algorithm~\ref{alg:dp-prs} (\NAMEA) applied to a private dataset $S$ of size $N$, with mini-batch size $B$, sampling rate $q = B/N$, $T$ training iterations, and per-sample $\ell_2$-gradient sensitivity $\Delta$ (from Lemma~\ref{lemma:appx:L2_sensitivity}).
    
    Let the desired total privacy guarantee be $(\varepsilon, \delta)$, for any $\delta \in (0,1)$ and for a sufficiently small $\varepsilon$ (specifically, $\varepsilon < c_1 q^2 T$ for a universal constant $c_1$). The algorithm guarantees $(\varepsilon, \delta)$-DP if the chosen explicit noise multiplier $\sigma$ satisfies:
    \begin{equation}
        \sigma^2 \ge  \max\left(0, C_{\text{DP}} - \frac{2B \sigma_{\mathrm{shot}}^2}{N_s (\lambda_{\max}-\lambda_{\min})^2} \right)
    \end{equation}
    where $C_{\text{DP}} = \left( c_2 \frac{q \sqrt{T \log(1/\delta)}}{\varepsilon} \right)^2$ with $c_2$ is a constant.
    \label{thm:appx:total_privacy}
\end{theorem}

\begin{proof}
    Based on the works~\cite{Balle2018_Subsampling,DPSGD}, for an Subsampled Gaussian Mechanism (SGM) composed over $T$ iterations with sampling rate $q$ to be $(\varepsilon, \delta)$-private, its per-iteration effective noise multiplier, $\sigma_{\text{total}}$, must satisfy:
    $$
        \sigma_{\text{total}}\ge c_2 \frac{q \sqrt{T \log(1/\delta)}}{\varepsilon}
    $$
    for at least one universal constant $c_2$.

    Each iteration of \NAMEA is an instance of the SGM. Its total noise comes from two independent sources: the explicitly added noise (with multiplier $\sigma$) and the implicit quantum shot noise. Based on Lemma~\ref{lemma:appx:iterationDP}, the total effective noise multiplier, $ \sigma_{\text{total}}$, is the sum of the squared individual multipliers:
    $$
    \sigma_{\text{total}}^2 = \sigma^2 + \frac{2B \sigma_{\mathrm{shot}}^2}{N_s (\lambda_{\max}-\lambda_{\min})^2}
    $$
    As a result, to ensure the algorithm is $(\varepsilon, \delta)$-private, we have:
    \begin{align*}
        \sigma^2 + \frac{2B \sigma_{\mathrm{shot}}^2}{N_s (\lambda_{\max}-\lambda_{\min})^2} \ge \left( c_2 \frac{q \sqrt{T \log(1/\delta)}}{\varepsilon} \right)^2
    \end{align*}
    Because $\sigma \ge 0$, we have $$\sigma^2 \ge  \max\left(0, \ \left( c_2 \frac{q \sqrt{T \log(1/\delta)}}{\varepsilon} \right)^2 - \frac{2B \sigma_{\mathrm{shot}}^2}{N_s (\lambda_{\max}-\lambda_{\min})^2} \right)$$.
\end{proof}

\begin{theorem}[Bound on the Gradient Estimation Error for \NAMEA]
    Let $\nabla C_S(\boldsymbol{\theta})$ be the true gradient over the full dataset $S$, and let $\tilde{\mathbf{g}}$ be the final noisy gradient estimate produced by Algorithm 1 (\NAMEA) for a mini-batch of size $B$. The Mean Squared Error (MSE) of the estimator is bounded by:
    \begin{equation}
        \mathbb{E}\left[ \|\tilde{\mathbf{g}} - \nabla C_S(\boldsymbol{\theta})\|^2 \right] \le \frac{\Delta^2}{B} \left( 1 + \frac{1}{2N_s} \right) + \frac{K \sigma^2 \Delta^2}{B^2}
    \end{equation}
    where $\Delta$ is the per-sample $\ell_2$-sensitivity from Lemma~\ref{lemma:appx:L2_sensitivity}, $N_s$ is the number of measurement shots, $K$ is the number of parameters, and $\sigma$ is the DP noise multiplier.
    \label{thm:appx:utility}
\end{theorem}

\begin{proof}
The total error can be decomposed into three independent, zero-mean sources: (1) artificial noise added for differential privacy, (2) variance from finite-shot estimation of gradients, and (3) variance from sampling a mini-batch instead of using the full dataset. We note that, in this proof, the symbol $B$ is used to denote both the mini-batch itself and its size.

Let $\tilde{\mathbf{g}}$ be the final noisy gradient, $\hat{\mathbf{g}}_B$ be the shot-noise-only average gradient for the batch $B$, and $\nabla C_B(\boldsymbol{\theta})$ be the true (exact) average gradient for the batch $B$. We can decompose the error vector as:
$$
\tilde{\mathbf{g}} - \nabla C_S(\boldsymbol{\theta}) = \underbrace{(\tilde{\mathbf{g}} - \hat{\mathbf{g}}_B)}_{\text{Artificial DP Noise}} + \underbrace{(\hat{\mathbf{g}}_B - \nabla C_B(\boldsymbol{\theta}))}_{\text{Estimator Variance}} + \underbrace{(\nabla C_B(\boldsymbol{\theta}) - \nabla C_S(\boldsymbol{\theta}))}_{\text{Sampling Variance}}
$$
We have these three error terms are zero-mean and independent, so the expected squared norm of the sum is the sum of their expected squared norms (the cross-terms disappear).
$$
\mathbb{E}\left[ \|\tilde{\mathbf{g}} - \nabla C_S(\boldsymbol{\theta})\|^2 \right] = \mathbb{E}\left[\|\tilde{\mathbf{g}} - \hat{\mathbf{g}}_B\|^2\right] + \mathbb{E}\left[\|\hat{\mathbf{g}}_B - \nabla C_B(\boldsymbol{\theta})\|^2\right] + \mathbb{E}\left[\|\nabla C_B(\boldsymbol{\theta}) - \nabla C_S(\boldsymbol{\theta})\|^2\right]
$$
We bound each term separately.

\textbf{Artificial DP Noise Error:}
The DP noise is $\tilde{\mathbf{g}} - \hat{\mathbf{g}}_B = \frac{1}{B}\mathbf{z}$, where $\mathbf{z} \sim \mathcal{N}(0, \sigma^2 \Delta^2 \mathbf{I}_K)$. Thus, we have:
$$
\mathbb{E}\left[ \|\tilde{\mathbf{g}} - \hat{\mathbf{g}}_B\|^2 \right] = \frac{1}{B^2} \mathbb{E}\left[ \|\mathbf{z}\|^2 \right] = \frac{1}{B^2} (K \cdot \sigma^2 \Delta^2) = \frac{K \sigma^2 \Delta^2}{B^2}
$$

\textbf{Data Sampling Variance:}
This term measures the error from using a batch of size $B$ instead of the full dataset of size $N$. We have:
\begin{align*}
    \mathbb{E}\left[ \|\nabla C_B - \nabla C_S\|^2 \right] &= \mathbb{E}\left[ \left\| \frac{1}{B}\sum_{j=1}^B \nabla C_j - \nabla C_S \right\|^2 \right]\\
    &= \mathbb{E}\left[\left\|\frac{1}{B}\sum_{j=1}^{B}(\nabla C_j - \nabla C_S)\right\|^2\right]\\ &= \frac{1}{B^2} \sum_{j=1}^{B} \mathbb{E}\left[\|\nabla C_j - \nabla C_S\|^2\right] \\
&= \frac{1}{B} \mathbb{E}\left[|\nabla C_j - \nabla C_S|^2\right] \\
&\le \frac{1}{B} \mathbb{E}\left[|\nabla C_j|^2\right] \text{(Because $\mathbb{E}[\nabla C_j] = \nabla C_S$)} \\
&\le  \frac{\Delta^2}{B} \text{(From Lemma~\ref{lemma:appx:L2_sensitivity})}
\end{align*}

\textbf{Estimator (Shot Noise) Variance:}
This term measures the error from estimating gradients with $N_s$ shots. We have:
$$
\mathbb{E}\left[ \|\hat{\mathbf{g}}_B - \nabla C_B\|^2 \right] = \mathbb{E}\left[ \left\| \frac{1}{B}\sum_{j=1}^B (\mathbf{g}^{(j)} - \nabla C_j) \right\|^2 \right] = \frac{1}{B^2} \sum_{j=1}^B \text{Var}(\mathbf{g}^{(j)}) = \frac{1}{B} \text{Var}(\mathbf{g})
$$
where we used the i.i.d. nature of the estimators. The variance of a single estimated gradient vector $\hat{\mathbf{g}}$ is $\sum_{k=1}^K \text{Var}(\hat{\mathbf{g}}_k)$. The variance of a single component is:
$$
\text{Var}(\mathbf{g}_k) = \text{Var}\left(\frac{\Omega_k}{2}(r_+ - r_-)\right) = \frac{\Omega_k^2}{4} (\text{Var}(r_+) + \text{Var}(r_-))
$$
The variance of the mean of $N_s$ shots is $\text{Var}(r_\pm) \le \frac{\sigma_{\text{shot}}^2}{N_s}$.
$$
\text{Var}(\mathbf{g}_k) \le \frac{\Omega_k^2}{4} \left( \frac{2\sigma_{\text{shot}}^2}{N_s} \right) = \frac{\Omega_k^2 \sigma_{\text{shot}}^2}{2 N_s}
$$
Summing over all components:
$$
\text{Var}(\mathbf{g}) = \sum_{k=1}^K \text{Var}(\mathbf{g}_k) \le \frac{\sigma_{\text{shot}}^2}{2N_s} \sum_{k=1}^K \Omega_k^2
$$
Using $\Delta = \frac{\lambda_{\max}-\lambda_{\min}}{2}\sqrt{\sum\Omega_k^2}$ and $\sigma_{\text{shot}}^2 \le \frac{(\lambda_{\max}-\lambda_{\min})^2}{4}$, we get $\sum\Omega_k^2 = \frac{4\Delta^2}{(\lambda_{\max}-\lambda_{\min})^2}$. Substituting these in:
$$
\text{Var}(\hat{\mathbf{g}}) \le \frac{(\lambda_{\max}-\lambda_{\min})^2/4}{2N_s} \cdot \frac{4\Delta^2}{(\lambda_{\max}-\lambda_{\min})^2} = \frac{\Delta^2}{2N_s}
$$
So, the total estimator variance is bounded by $\frac{1}{B} \text{Var}(\mathbf{g}) \le \frac{\Delta^2}{2BN_s}$.

Combining the three error bounds gives the total MSE:
$$
\mathbb{E}\left[ \|\tilde{\mathbf{g}} - \nabla C_S(\boldsymbol{\theta})\|^2 \right] \le \frac{K \sigma^2 \Delta^2}{B^2} + \frac{\Delta^2}{B} + \frac{\Delta^2}{2BN_s} = \frac{\Delta^2}{B} \left( 1 + \frac{1}{2N_s} \right) + \frac{K \sigma^2 \Delta^2}{B^2}
$$
\end{proof}

\begin{theorem}[Lower Bound on Single-Shot Variance with Depolarizing Noise]
    The single-shot variance of a measurement of an observable $O$ on this noisy state, $\eta_{\mathrm{shot}}^2(x)$, is lower-bounded by:
    \begin{equation}
        \eta_{\mathrm{shot}}^2(x) \ge \alpha \cdot \sigma_{\mathrm{uniform}}^2
    \end{equation}
    where $\sigma_{\mathrm{uniform}}^2$ is the variance of the observable with respect to the maximally mixed state ($I/d$), a constant value given by:
    $$
    \sigma_{\mathrm{uniform}}^2 = \left(\frac{1}{d}\sum_i \lambda_i^2\right) - \left(\frac{1}{d}\sum_i \lambda_i\right)^2
    $$
    \label{thm:appx:shot_variance_lower_bound}
\end{theorem}

\begin{proof}
Consider a quantum circuit that prepares an ideal state $\rho_{\text{ideal}}(x) = \ket{\psi(x)}\bra{\psi(x)}$. Suppose this state passes through a global depolarizing channel $\mathcal{E}_{\alpha}$ of strength $\alpha \in [0, 1]$ before measurement. The final (noisy) state is:
    $$
    \rho_{\text{noisy}}(x) = \mathcal{E}_{\alpha}(\rho_{\text{ideal}}(x)) = (1-\alpha)\rho_{\text{ideal}}(x) + \alpha \frac{I}{d}
    $$
    where $I$ is the identity operator and $d$ is the dimension of the Hilbert space.

First, we calculate the probability distribution of outcomes for the noisy state and then apply the definition of variance, $\eta_{\mathrm{shot}}^2(x) = \mathbb{E}[\hat{a}^2] - (\mathbb{E}[\hat{a}])^2$.

Let $p_i^{\text{ideal}}(x) = \text{Tr}(\rho_{\text{ideal}}(x) P_i)$ be the outcome probability for the ideal state. The probability of measuring outcome $\lambda_i$ on the noisy state $\rho_{\text{noisy}}(x)$ is:
\begin{align*}
    p_i^{\text{noisy}}(x) &= \text{Tr}(\rho_{\text{noisy}}(x) P_i) \\
    &= \text{Tr}\left( \left((1-\alpha)\rho_{\text{ideal}}(x) + \alpha \frac{I}{d}\right) P_i \right) \\
    &= (1-\alpha)\text{Tr}(\rho_{\text{ideal}}(x) P_i) + \frac{\alpha}{d}\text{Tr}(I P_i) \\
    &= (1-\alpha)p_i^{\text{ideal}}(x) + \frac{\alpha}{d} 
\end{align*}
where we have used $\text{Tr}(P_i)=1$ for a rank-1 projector. This shows that the noisy probability distribution is a convex combination of the ideal and uniform distributions.

Let $\mathbb{E}[\cdot]_{\text{noisy}}$, $\mathbb{E}[\cdot]_{\text{ideal}}$, and $\mathbb{E}[\cdot]_{\text{uniform}}$ denote expectations with respect to the noisy, ideal, and uniform distributions. The expectation values are linear combinations:
\begin{align*}
    \mathbb{E}[\hat{a}]_{\text{noisy}} &= \sum_i \lambda_i p_i^{\text{noisy}}(x) = (1-\alpha)\mathbb{E}[\hat{a}]_{\text{ideal}} + \alpha\mathbb{E}[\hat{a}]_{\text{uniform}} \\
    \mathbb{E}[\hat{a}^2]_{\text{noisy}} &= \sum_i \lambda_i^2 p_i^{\text{noisy}}(x) = (1-\alpha)\mathbb{E}[\hat{a}^2]_{\text{ideal}} + \alpha\mathbb{E}[\hat{a}^2]_{\text{uniform}}
\end{align*}
Substituting these into the variance formula for the noisy state, $\eta_{\mathrm{shot}}^2(x) = \mathbb{E}[\hat{a}^2]_{\text{noisy}} - (\mathbb{E}[\hat{a}]_{\text{noisy}})^2$, we have:
\begin{align*}
    \eta_{\mathrm{shot}}^2(x) &= \left( (1-\alpha)\mathbb{E}[\hat{a}^2]_{\text{ideal}} + \alpha\mathbb{E}[\hat{a}^2]_{\text{uniform}} \right) - \left( (1-\alpha)\mathbb{E}[\hat{a}]_{\text{ideal}} + \alpha\mathbb{E}[\hat{a}]_{\text{uniform}} \right)^2 \\
    % Expand the squared term
    &= (1-\alpha)\mathbb{E}[\hat{a}^2]_{\text{ideal}} + \alpha\mathbb{E}[\hat{a}^2]_{\text{uniform}} \\
    & \quad - \left( (1-\alpha)^2(\mathbb{E}[\hat{a}]_{\text{ideal}})^2 + \alpha^2(\mathbb{E}[\hat{a}]_{\text{uniform}})^2
    + 2\alpha(1-\alpha)\mathbb{E}[\hat{a}]_{\text{ideal}}\mathbb{E}[\hat{a}]_{\text{uniform}} \right) \\
\end{align*}
Recall that $\sigma^2 = \mathbb{E}[\hat{a}^2] - (\mathbb{E}[\hat{a}])^2$, so by replacing $\mathbb{E}[\hat{a}^2] = \sigma^2 + (\mathbb{E}[\hat{a}])^2$, we have:
\begin{align*}
    \eta_{\mathrm{shot}}^2(x) &= (1-\alpha)\sigma_{\text{ideal}}^2(x) + \alpha\sigma_{\text{uniform}}^2 + (1-\alpha)(\mathbb{E}[\hat{a}]_{\text{ideal}})^2 - (1-\alpha)^2(\mathbb{E}[\hat{a}]_{\text{ideal}})^2 \\
    & \quad + \alpha(\mathbb{E}[\hat{a}]_{\text{uniform}})^2 - \alpha^2(\mathbb{E}[\hat{a}]_{\text{uniform}})^2 - 2\alpha(1-\alpha)\mathbb{E}[\hat{a}]_{\text{ideal}}\mathbb{E}[\hat{a}]_{\text{uniform}}\\
    &= (1-\alpha)\sigma_{\text{ideal}}^2(x) + \alpha\sigma_{\text{uniform}}^2 + \alpha(1-\alpha)\left( (\mathbb{E}[\hat{a}]_{\text{ideal}})^2 - 2\mathbb{E}[\hat{a}]_{\text{ideal}}\mathbb{E}[\hat{a}]_{\text{uniform}} + (\mathbb{E}[\hat{a}]_{\text{uniform}})^2 \right) \\
    % Recognize the perfect square
    &= (1-\alpha)\sigma_{\text{ideal}}^2(x) + \alpha\sigma_{\text{uniform}}^2 + \alpha(1-\alpha)(\mathbb{E}[\hat{a}]_{\text{ideal}}(x) - \mathbb{E}[\hat{a}]_{\text{uniform}})^2
\end{align*}

Since all terms in the above equation are non-negative, we can drop the first and third terms to deliver the inequality:
$$
\eta_{\mathrm{shot}}^2(x) \ge \alpha \cdot \sigma_{\text{uniform}}^2
$$
\end{proof}

\begin{theorem}
\label{thm:appx:advanced_adaptive_privacy}
For a mini-batch of size $B$, let $\hat{\bm{\eta}}_B^2$ be a probabilistic estimator for the true sum of all variances, $\bm{\eta}_B^2 = \sum_{j=1}^{B} \sum_{k=1}^K \sum_{\tau \in \pm} \eta_{j,k,\tau}^2$ such that it satisfies $\Pr(\bm{\eta}_B^2 \ge \hat{\bm{\eta}}_B^2) \to 1-\beta$ for a chosen significance level $\beta$.
\NAMEB calculates a noise multiplier for each batch $B$, $\sigma^2_B$, satisfying:
$$
\sigma^2_B \ge \max\left(0, C_{\text{DP}} - \frac{\Omega^2\hat{\bm{\eta}}_{B}^2}{4N_s \Delta^2} \right)
$$
where $C_{\text{DP}} = \ \left( c_2 \frac{q \sqrt{T \log(1/\delta)}}{\varepsilon} \right)^2$ with $c_2$ is a constant. \NAMEB achieves \textbf{$(\varepsilon, (1-\beta)\delta + \beta)$-DP}.
\end{theorem}

The proof of this result rely on supporting lemmas, which we state and prove below.

\begin{lemma}[Multi-Group Variance Estimation]
\label{lemma:pivot_quantitity}
Let there be $L$ independent groups, indexed by $i=1, \dots, L$. For each group $i$, let $y_{i,1}, \dots, y_{i,n} \sim \text{i.i.d. } Y_i$ be $n$ observations. Let $\eta_i^2$, $\bar{\eta}_i^2$, $\mu_{4,i}$ and $\bar{\mu}_{4,i}$ be the true variance, sample variance, true fourth central moment, and sample fourth central moment for the $i$-th group. Let $\eta^2_{\textrm{batch}} = \sum_{i=1}^L \eta^2_i$ and $\bar{\eta}_{\textrm{batch}}^2 = \sum_{i=1}^L \bar{\eta}_i^2$. Assuming all $\mu_{4,i}$ are finite and given a significance level $\beta$ with critical value $z_\beta$, then:
$$
P\left(\eta^2_{\textrm{batch}} \geq \bar{\eta}^2_{\textrm{batch}} - z_\beta\sqrt{\sum_{i=1}^L \frac{1}{n}\left(\bar{\mu}_{4,i} - (\bar{\eta}_i^2)^2\right)}\right) \to 1 - \beta
$$
as $L \to \infty, n \to \infty$.
\end{lemma}

\begin{comment}
    \begin{lemma} 
Let there be $L$ independent groups, indexed by $i=1, \dots, L$. For each group $i$, let $y_{i,1}, \dots, y_{i,K} \sim \text{i.i.d. } Y_i$ be $n$ observations. Let $\bar{\eta}_i^2$ and $\bar{\mu}_{4,i}$ be the sample variance and sample fourth central moment for the $i$-th group. Let the true sum of variances be $\eta^2_{\textrm{batch}} = \sum_{i=1}^L \eta^2_i$ and its estimator be $\bar{\eta}_{\textrm{batch}}^2 = \sum_{i=1}^L \bar{\eta}_i^2$. Assuming each true fourth central moment $\mu_{4,i}$ is finite, then:
$$
P\left(\eta^2_{\textrm{batch}} \geq \bar{\eta}^2_{\textrm{batch}} - z_\beta\sqrt{\sum_{i=1}^L \frac{1}{n}\left(\bar{\mu}_{4,i} - (\bar{\eta}_i^2)^2\right)}\right) \to 1 - \beta
$$
as $L \to \infty, n \to \infty$.
\end{lemma}
\end{comment}

\begin{proof}
First, we prove a sublemma for single-group case as follow:
\begin{sublemma} [Single-Group Variance Estimation]
\label{lemma:pivot_quantitity_single_group}
    Let $y_1, \dots, y_{n}$ be independent and identically distributed (i.i.d.) samples drawn from a distribution with true variance $\eta^2$ and a finite true fourth central moment $\mu_4$. Let $\bar{\eta}^2$ and $\bar{\mu}_{4}$ be the sample variance and sample fourth central moment, respectively, computed from these observations. For a chosen significance level $\beta$, let $z_\beta$ be the corresponding critical value of the standard normal distribution.

Then, the following probabilistic lower bound holds:
$$
\Pr\left(\eta^2 \geq \bar{\eta}^2 - z_\beta\sqrt{\frac{1}{n}\left(\bar{\mu}_{4} - (\bar{\eta}^2)^2\right)}\right) \to 1 - \beta
$$
as the number of samples $n \to \infty$.
\end{sublemma}

\begin{proof}
The proof is based on the asymptotic distribution of the sample variance. Let the true variance be $\eta^2$ and the true fourth central moment be $\mu_4$. The corresponding sample estimators are $\bar{\eta}^2$ and $\bar{\mu}_4$.

By the Central Limit Theorem and Slutsky's Theorem, the standardized statistic $T_n$ is known to converge in distribution to a standard normal random variable $Z \sim \mathcal{N}(0,1)$:
$$
T_n = \frac{\sqrt{n}(\bar{\eta}^2 - \eta^2)}{\sqrt{\bar{\mu}_{4} - (\bar{\eta}^2)^2}} \xrightarrow{d} Z
$$
By definition of the one-sided critical value $z_\beta$, we know that for a standard normal variable $Z$, $\Pr(Z \leq z_\beta) = 1-\beta$. Since $T_n$ converges in distribution to $Z$, we can state the asymptotic probability:
$$
\Pr(T_n \leq z_\beta) \to 1 - \beta \quad \text{as } n \to \infty
$$
Substituting the definition of $T_n$, we have:
$$
\Pr\left(\frac{\sqrt{n}(\bar{\eta}^2 - \eta^2)}{\sqrt{\bar{\mu}_{4} - (\bar{\eta}^2)^2}} \leq z_\beta\right) \to 1 - \beta
$$
We now algebraically manipulate the inequality inside the probability to isolate the true parameter $\eta^2$. Each of the following steps is equivalent.
\begin{align*}
    \frac{\sqrt{n}(\bar{\eta}^2 - \eta^2)}{\sqrt{\bar{\mu}_{4} - (\bar{\eta}^2)^2}} &\leq z_\beta \\
    \bar{\eta}^2_{n} - \eta^2 &\leq \frac{z_\beta}{\sqrt{n}}\sqrt{\bar{\mu}_{4} - (\bar{\eta}^2)^2} \\
    -\eta^2 &\leq -\bar{\eta}^2 + \frac{z_\beta}{\sqrt{n}}\sqrt{\bar{\mu}_{4} - (\bar{\eta}^2)^2} \\
    \eta^2 &\geq \bar{\eta}^2 - \frac{z_\beta}{\sqrt{n}}\sqrt{\bar{\mu}_{4} - (\bar{\eta}^2)^2}
\end{align*}
We can rewrite the final term by bringing the $1/\sqrt{n}$ factor inside the square root:
$$
\eta^2 \geq \bar{\eta}^2 - z_\beta\sqrt{\frac{1}{n}\left(\bar{\mu}_{4} - (\bar{\eta}^2)^2\right)}
$$
Since this inequality is algebraically equivalent to our starting point ($T_n \leq z_\beta$), the probability of it being true must have the same limit. Therefore, we have proven the statement in the sublemma.
\end{proof}

The proof extends the single-group case by applying the Central Limit Theorem to the sum of the sample variances, $\bar{\eta}_{\textrm{batch}}^2 = \sum_{i=1}^L \bar{\eta}_i^2$.

From the original lemma, we know the asymptotic variance of each individual sample variance $\bar{\eta}_i^2$ as $n \to \infty$ is given by:
\[
\text{Var}(\bar{\eta}_i^2) \approx \frac{1}{n}(\mu_{4,i} - (\eta_i^2)^2)
\]
We consider $\bar{\eta}_{\textrm{batch}}^2$ as a sum of $L$ independent random variables ($\bar{\eta}_1^2, \dots, \bar{\eta}_L^2$). By the Central Limit Theorem, for a large number of groups $L$, the distribution of $\bar{\eta}_{\textrm{batch}}^2$ approaches a normal distribution. The mean of this sum is the sum of the means, $E[\bar{\eta}_{\textrm{batch}}^2] = \sum_{i=1}^L E[\bar{\eta}_i^2] = \sum_{i=1}^L \eta_i^2 = \eta^2_{\textrm{batch}}$. The variance of the sum is the sum of the variances, $\text{Var}(\bar{\eta}_{\textrm{batch}}^2) = \sum_{i=1}^L \text{Var}(\bar{\eta}_i^2)$.

Thus, the standardized sum converges in distribution to a standard normal variable $Z \sim \mathcal{N}(0,1)$. By Slutsky's Theorem, we can replace the true parameters in the variance term with their consistent sample estimators without changing the asymptotic distribution. This gives the test statistic $T_L$:
\[
T_L = \frac{\bar{\eta}^2_{\textrm{batch}} - \eta^2_{\textrm{batch}}}{\sqrt{\sum_{i=1}^L \frac{1}{n}\left(\bar{\mu}_{4,i} - (\bar{\eta}_i^2)^2\right)}} \xrightarrow{d} Z
\]
By definition of the critical value $z_\beta$, we have $\Pr(T_L \leq z_\beta) \to 1 - \beta$ as $L \to \infty$. We now isolate $\eta^2_{\textrm{batch}}$ within the probability statement.
\begin{align*}
    \frac{\bar{\eta}^2_{\textrm{batch}} - \eta^2_{\textrm{batch}}}{\sqrt{\sum_{i=1}^L \frac{1}{n}\left(\bar{\mu}_{4,i} - (\bar{\eta}_i^2)^2\right)}} &\leq z_\beta \\
    \bar{\eta}^2_{\textrm{batch}} - \eta^2_{\textrm{batch}} &\leq z_\beta \sqrt{\sum_{i=1}^L \frac{1}{n}\left(\bar{\mu}_{4,i} - (\bar{\eta}_i^2)^2\right)} \\
    -\eta^2_{\textrm{batch}} &\leq -\bar{\eta}^2_{\textrm{batch}} + z_\beta \sqrt{\sum_{i=1}^L \frac{1}{n}\left(\bar{\mu}_{4,i} - (\bar{\eta}_i^2)^2\right)} \\
    \eta^2_{\textrm{batch}} &\geq \bar{\eta}^2_{\textrm{batch}} - z_\beta \sqrt{\sum_{i=1}^L \frac{1}{n}\left(\bar{\mu}_{4,i} - (\bar{\eta}_i^2)^2\right)}
\end{align*}
Since the final inequality is algebraically equivalent to $T_L \leq z_\beta$, its probability must also converge to $1-\beta$, which completes the proof.
\end{proof}

\begin{lemma}
\label{lemma:mechanism_probabilistic_dp}
Let $\mathcal{M}$ be a randomized mechanism that operates in one of two modes. With probability $\gamma$, it operates in a ``Good Mode,'' where its execution is equivalent to a mechanism $\mathcal{M}_{good}$ that is $(\varepsilon, \delta)$-differentially private. With probability $1-\gamma$, it operates in a ``Bad Mode,'' where no privacy guarantee holds. Then, the overall mechanism $\mathcal{M}$ is $(\varepsilon, \gamma\delta + (1-\gamma))$-differentially private.
\end{lemma}

\begin{proof}
Let $x$ and $y$ be any two adjacent databases, and let $S$ be any set of possible outcomes. The mechanism's choice of mode is determined by its internal randomness, which is resampled for each run. Therefore, the events of being in a good mode are independent for the runs on $x$ and $y$.

Let $E_{good, x}$ be the event that the mechanism operates in the Good Mode when run on database $x$.
Let $E_{good, y}$ be the event that the mechanism operates in the Good Mode when run on database $y$.
We are given $\Pr(E_{good, x}) = \Pr(E_{good, y}) = \gamma$.

By the law of total probability, we can write the probability of an output for database $x$:
\begin{equation}
\Pr[\mathcal{M}(x) \in S] = \Pr[\mathcal{M}(x) \in S | E_{good, x}] \cdot \gamma + \Pr[\mathcal{M}(x) \in S | \neg E_{good, x}] \cdot (1-\gamma)
\label{eq:total_prob_x}
\end{equation}

We now bound the two terms on the right-hand side. The mechanism that runs conditioned on the Good Event is, by assumption, the same $(\varepsilon, \delta)$-DP mechanism, $\mathcal{M}_{good}$, regardless of the input. Thus, its behavior on $x$ (conditioned on $E_{good,x}$) can be compared to its behavior on $y$ (conditioned on $E_{good,y}$):
\[
\Pr[\mathcal{M}(x) \in S | E_{good, x}] \le e^{\varepsilon} \Pr[\mathcal{M}(y) \in S | E_{good, y}] + \delta
\]
For the Bad Event, we can only bound the probability by 1:
\[
\Pr[\mathcal{M}(x) \in S | \neg E_{good, x}] \le 1
\]

Substituting these bounds into Equation \ref{eq:total_prob_x}:
\begin{align}
\Pr[\mathcal{M}(x) \in S] &\le \left( e^{\varepsilon} \Pr[\mathcal{M}(y) \in S | E_{good, y}] + \delta \right) \cdot \gamma + 1 \cdot (1-\gamma) \nonumber \\
&\le \gamma \cdot e^{\varepsilon} \Pr[\mathcal{M}(y) \in S | E_{good, y}] + \gamma\delta + (1-\gamma) \label{eq:bound1}
\end{align}

To complete the proof, we relate the conditional probability for database $y$ back to the total probability. By the law of total probability for the run on $y$:
\[
\Pr[\mathcal{M}(y) \in S] = \Pr[\mathcal{M}(y) \in S | E_{good, y}] \cdot \gamma + \Pr[\mathcal{M}(y) \in S | \neg E_{good, y}] \cdot (1-\gamma)
\]
Since probabilities are non-negative, it follows that:
\[
\Pr[\mathcal{M}(y) \in S] \ge \Pr[\mathcal{M}(y) \in S | E_{good, y}] \cdot \gamma
\]
\[
\implies \Pr[\mathcal{M}(y) \in S | E_{good, y}] \le \frac{\Pr[\mathcal{M}(y) \in S]}{\gamma}
\]

Finally, we substitute this back into our main inequality (Equation \ref{eq:bound1}):
\begin{align*}
\Pr[\mathcal{M}(x) \in S] &\le \gamma \cdot e^{\varepsilon} \left( \frac{\Pr[\mathcal{M}(y) \in S]}{\gamma} \right) + \gamma\delta + (1-\gamma) \\
&\le e^{\varepsilon} \Pr[\mathcal{M}(y) \in S] + \gamma\delta + (1-\gamma)
\end{align*}

This expression matches the definition of $(\varepsilon, \gamma\delta + (1-\gamma))$-differential privacy.
\end{proof}

We are now ready to prove the main theorem for the advanced adaptive mechanism, \NAMEB.
The privacy of the algorithm for a given batch depends on an aggregated estimator, $\hat{\bm{\eta}}_B^2$, which provides a probabilistic lower bound for the true sum of all shot-noise variances $\bm{\eta}_B^2$ in that batch.

The algorithm is in a ``Good Mode'' if this estimator $\hat{\bm{\eta}}_B^2$ is a correct lower bound on $\bm{\eta}_B^2$. Lemma~\ref{lemma:pivot_quantitity} directly provides the probability for this single event:
$$
\operatorname{Pr}_{\text{good}} = \Pr(\bm{\eta}_B^2 \ge \hat{\bm{\eta}}_B^2) \to 1-\beta
$$
Consequently, the algorithm is in a ``Bad Mode'' with the remaining probability, $\operatorname{Pr}_{\text{bad}} = 1 - \Pr_{\text{good}} \to \beta$.

By applying Lemma ~\ref{lemma:mechanism_probabilistic_dp} with these probabilities, we conclude that the adaptive mechanism is $(\varepsilon, \Pr_{\text{good}} \cdot \delta + \Pr_{\text{bad}})$-DP. This yields the final privacy guarantee of:
$$
(\varepsilon, (1-\beta)\delta + \beta)\text{-DP}
$$

\section{Related Works}
\label{sec:related_works}
In this section, we first review the literature on differential privacy in classical machine learning. Then, we discuss existing works that extend this concept to quantum machine learning.

\subsection{Differential Privacy in Classical Machine Learning}

The goal of DP in Classical ML is to modify the ML training process so the final model is differentially private, i.e, the model would not be sufficiently different, no matter whether a particular instance was or was not included in the training data. Early works focused on modifying already trained model weights, a strategy often called output perturbation. Originating from Dwork’s definition of differential privacy \cite{dwork2006differential}, these methods add noise proportional to the sensitivity of the final model parameters. However, computing tight sensitivity bounds is tractable only for relatively simple models such as linear regression due to the complex dependence of weights on the data \cite{zhang2012functional}. For example, \cite{chaudhuri2008privacy} proposed an algorithm that injects calibrated noise into trained logistic regression models to achieve differential privacy.

Another line of work achieves DP-Training by perturbing the loss function itself with noise, which is then optimized normally using SGD or other optimizers \cite{chaudhuri2011differentially, kifer2012private, phan2016differential}. Both noising weights and objective perturbation generally require strong convexity assumptions and convergence to a global optimum to guarantee privacy. Under these assumptions they can provide very strong privacy ($\varepsilon \approx 1$) with minimal accuracy loss. Subsequent studies have attempted to relax the requirements of global optimality \cite{iyengar2019towards} or convexity \cite{neel2019differentially}, but their privacy guarantees still depend on the loss function being bounded and Lipschitz continuity in the model weights. These methods remain computationally expensive and are therefore practical only for relatively small datasets.

Because most deep models are non-convex and are not trained to the global optimality due to time constraints and the computing cost, such loss perturbation techniques are rarely practical for modern deep learning. Gradient perturbation based methods are currently the most practical approach for ensuring rigorous privacy guarantees in non-convex settings such as large scale deep neural networks. The most widely adopted variant is differentially private stochastic gradient descent (DP-SGD) \cite{song2013stochastic,abadi2016deep}, which has become the standard method for differentially private training. Briefly, DP-SGD clips each per example gradient to a fixed norm, adds Gaussian noise to the averaged gradients, and tracks the cumulative privacy loss using accounting techniques. Recent advances in DP-SGD focus on tightening privacy accounting bounds \cite{mironov2017renyi,dong2022gaussian}, developing adaptive clipping techniques \cite{chen2020understanding,he2023exploring} to improve the noise–utility trade-off, and introducing memory-efficient algorithms \cite{li2021large,tang2024private} to mitigate the large memory footprint caused by gradient accumulation. 
\subsection{Differential Privacy in Quantum Setting}

With the growing popularity of quantum computing, the concept of differential privacy has naturally been extended to quantum machine learning (QML). The first formal notion of quantum differential privacy (QDP) was introduced by \cite{ZhouQDPfirst}, who defined it as a direct quantum analogue of classical differential privacy. Building on this foundation, \cite{Du2021Depolarizing} explored its practical relevance to QML. Specifically, they illustrate that intrinsic quantum noise can be harnessed to achieve QDP in quantum classifiers. For example, they study the depolarizing noise channel as a privacy-preserving mechanism, and then derive a precise relationship between the noise strength and the resulting $(\varepsilon, \delta)$-QDP guarantee. Notably, they also established that this mechanism not only preserves the differential privacy but simultaneously improves the model’s adversarial robustness.

Subsequently, \cite{Hirche2023InfoTheoryQDP} advanced the field by developing a comprehensive information-theoretic framework for QDP. In this work, the authors leverage tools such that quantum relative entropy and channel divergence to establish a rigorous and general foundation for analyzing privacy guarantees in quantum settings. This framework not only formalizes the definition of QDP but also connects it with fundamental quantities in quantum information theory. Building on these foundations, more recent studies \cite{bai2024quantum,watkins2023quantum,song2025towards} have turned their attention to the practical implications of QDP under realistic quantum noise. In particular, these works investigate how different types of quantum noise channels, such as depolarizing, bit-flip, and phase-flip channels, contribute to the privacy budget.

Most existing works in quantum differential privacy (QDP) focus on injecting quantum noise into the forward pass to preserve privacy, which is conceptually equivalent to input perturbation in classical DP~\cite{Lcuyer2018CertifiedRT}. However, in classical neural networks, perturbing the forward pass (i.e., inputs or outputs) often results in higher utility degradation than perturbing model's gradients~\cite{Jarin2022_InOutGradientPerturb}. The same limitation may arise in QDP. Recently, a few studies~\cite{YenChi2021,YenChi2025} have explored adding noise to quantum gradients. But, these approaches remain largely ad hoc, as they directly apply DP-SGD to quantum gradients without accounting for their unique properties, such as boundedness and inherent stochasticity. This can lead to injecting an excessive amount of artificial DP noise. Therefore, it is crucial to design methods that exploit the structural characteristics of quantum gradients to achieve stronger privacy–utility trade-offs.

\section{Experimental Settings}\label{appx:extra_exp}
% Note: Since clipping norm is fixed, we can run sweep on Learning rate and batch size and epsilon. Also, for ablation, we can test with different optimizer (Adam, SGD, ...) and privacy accountant. Also, is there a naive baseline method we can compare to? Maybe a straightforward quantum DP method that adds a lot of noise. Plot convergence rate. Plot fairness ( biased dataset)?
\subsection{Datasets and Quantum Models.}\label{appx:dataset_model}
We assess the performance and privacy–utility trade-off of Q-ShiftDP on three representative datasets from recent QML benchmark suites \cite{bowles2024better, recio2025train} — Bars \& Stripes, Binary Blobs, and Downscaled MNIST, which are specifically designed to offer a controlled and reliable test bed for evaluating quantum models. The Bars and Stripes dataset is a simple, translation-invariant benchmark consisting of 2D pixel grids representing either horizontal bars or vertical stripes. %To add variability, independent Gaussian noise with a standard deviation of 0.5 is added to every pixel. 
The Binary Blobs dataset is a bitstring version of the common Gaussian blobs dataset. Each sample is a 16-bit string generated by randomly selecting one of eight predefined bit patterns and independently flipping each bit with a 5\% probability. This process creates a distribution with eight clearly separated, visually distinct classes. Finally, Downscaled MNIST provides a reduced-dimension digit classification task derived from the original MNIST dataset, preserving class-discriminative information while keeping the input size tractable for quantum circuits. Our classifier is a fully quantum 4-qubit neural network (QNN) implemented in PennyLane and trained using PyTorch. For Bars \& Stripes, we employ a single Strongly Entangling Layer, while for Binary Blobs and Downscaled MNIST, we use a five-layer Strongly Entangling Layers ansatz. The first two basis probabilities are used for binary classification, with their log-probabilities passed to the NLL loss. This is equivalent to using two binary observables $O_0$ and $O_1$, one for each class, which simplifies implementation. %Gradients are then privatized using \NAMEA.
\subsection{Impact of Hyperparameters on the Performance of \NAMEA}
\label{appx:hyperparam_tuning}
Differentially private stochastic optimizers are known to be highly sensitive to hyperparameter selection \cite{bu2023automatic, sander2023tan}; even slight adjustments to the clipping norm or batch size can lead to substantial performance changes. In contrast to DP-SGD which must clip unbounded per-example gradients to guarantee privacy, \NAMEA exploits the parameter-shift rule and the bounded spectrum of quantum observables to compute an exact analytic sensitivity bound (see Lemma \ref{lemma:appx:L2_sensitivity}), removing the need to manually tune a clipping norm. To investigate how batch size and learning rate affect Q-ShiftDP’s performance, we perform a grid search over batch sizes $\{32,64,128,256,512\}$, learning rates $\{0.2,0.1,0.05,0.01,0.005\}$, and privacy budgets $\varepsilon\in\{0.1,0.5,1.0\}$. For simplicity, this analysis assumes analytic expectation values rather than a finite number of measurement shots.

\begin{figure}[h!]
    \centering
    \includegraphics[width=1\linewidth]{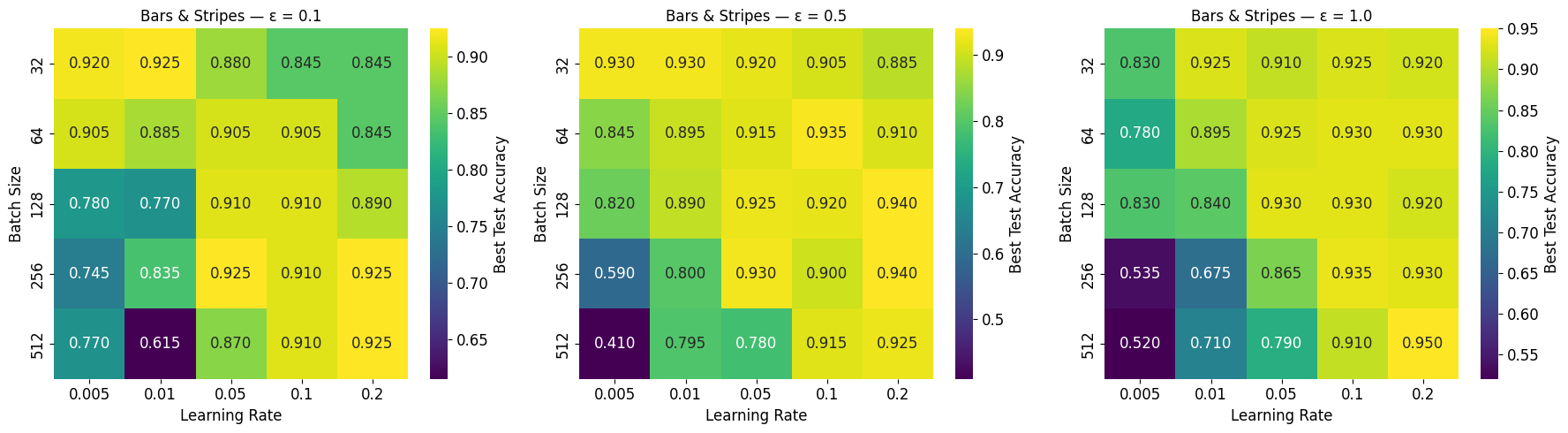}
    \caption{ Test accuracy on Bars and Stripes dataset using QNN with gradient privated under Q-ShiftDP by different batch size and learning rates.}
    \label{fig:bar_stripe}
\end{figure}

\begin{figure}[h!]
    \centering
    \includegraphics[width=1\linewidth]{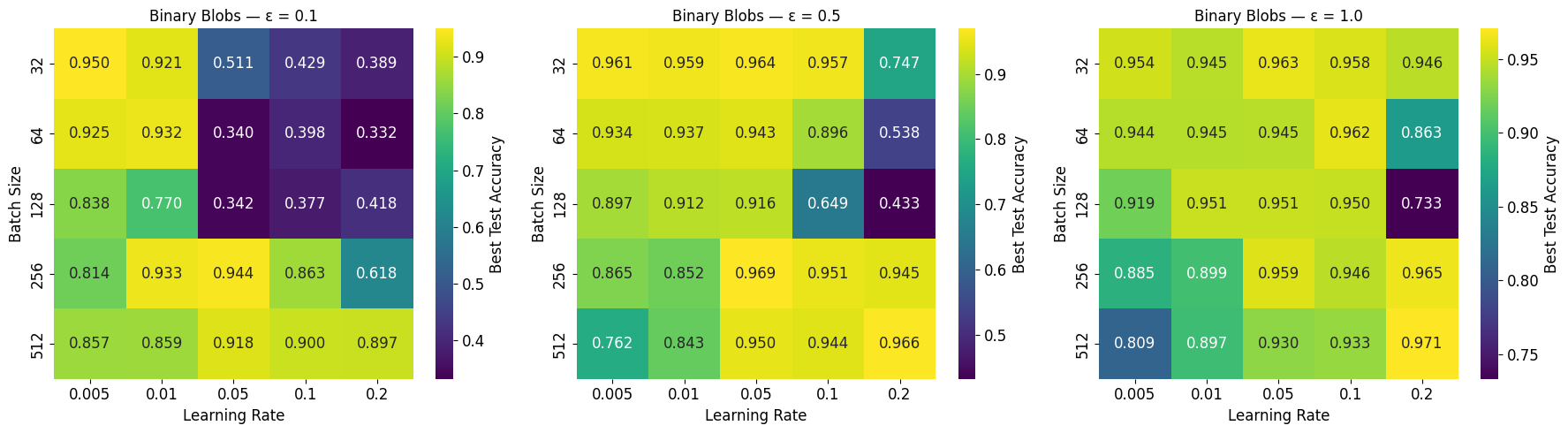}
    \caption{ Test accuracy  on Binary Blobs dataset of QNN with gradient privated using Q-ShiftDP by different batch size and learning rates}
    \label{fig:binaryblob}
\end{figure}

\begin{figure}[h!]
    \centering
    \includegraphics[width=1\linewidth]{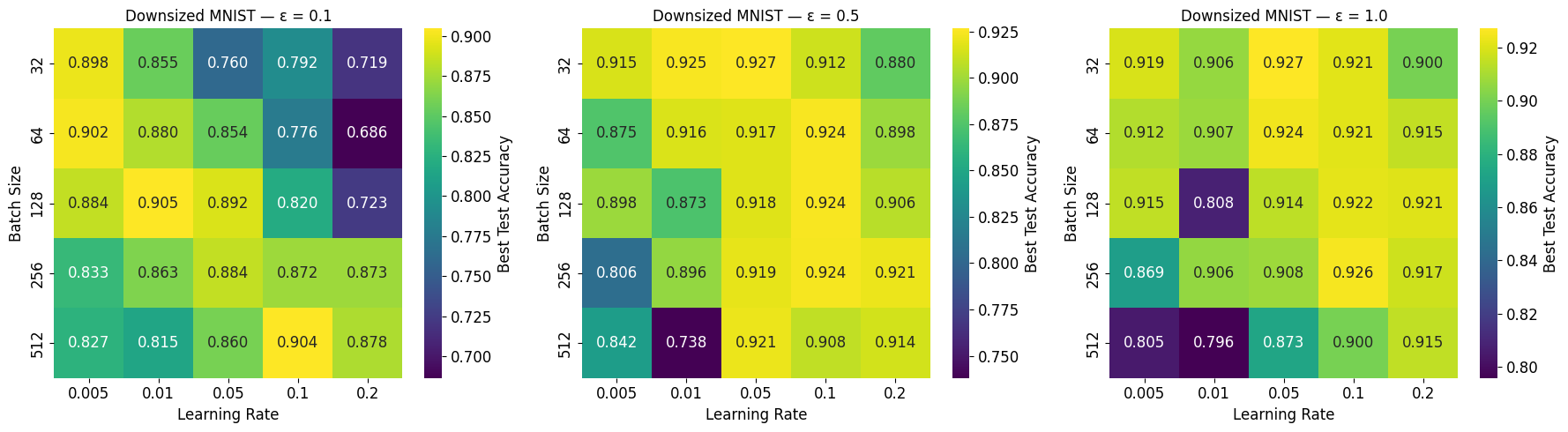}
    \caption{ Test accuracy on Downsized MNIST dataset of QNN with gradient privated using Q-ShiftDP by different batch size and learnifng rates}
    \label{fig:mnist}
\end{figure}

Figure \ref{fig:bar_stripe}, \ref{fig:binaryblob} and \ref{fig:mnist} shows that privacy budget, batch size, and learning rate strongly influence DP training on all datasets. Higher privacy budgets (larger $\varepsilon$) lead to better accuracy, while very small values ($\approx 0.1$) hurt performance. Large batche size usually perform best when privacy is strong ($\varepsilon = 0.1$ or 0.5), especially when paired with a large enough learning rate. This is consistent with prior DP-SGD literature showing that larger batch sizes yield more stable training across most privacy budgets \cite{de2022unlocking}. For the remaining experiments, we use a batch size of $B=512$ and a learning rate of $\text{lr}=0.2$, as this combination remains stable across all tested privacy budgets and datasets.

\subsection{Impact of number of \emph{shots} on performance of Q-ShiftDP}
\label{appx:shots}

\begin{table}[ht]
\centering
\begin{tabular}{ccccc}
\toprule
$\varepsilon$ & 1{,}000 shots & 10{,}000 shots & 100{,}000 shots & $\infty$ shots \\
\midrule
1 & 0.83 & 0.91 & 0.91 & 0.950 \\
0.5 & 0.82 & 0.90 & 0.90 & 0.925 \\
0.1   & 0.81 & 0.86 & 0.89 & 0.925 \\
\bottomrule
\end{tabular}
\caption{Accuracy of the QNN model on Bars and Stripes dataset for different privacy budgets $\varepsilon$ and numbers of shots. The $\infty$ column represents the ideal case with no sampling noise.}
\label{tab:dp_shots}
\end{table}

Table \ref{tab:dp_shots} shows the impact of number of \emph{shots} on the performace of Q-ShiftDP. We observe that increasing the number of shots consistently improves model accuracy across all privacy budgets. With only $1{,}000$ shots the gradient estimates are noisy, leading to reduced performance, while using $100{,}000$ or infinitely many shots brings the accuracy closer to the ideal non-sampling case. This effect is especially important when the privacy budget is tight (small $\varepsilon$), where the added stochasticity from few shots compounds with the differential privacy noise and harms learning.

%\subsection{Comparison to input noise perturbation baseline (PixelDP)}
%\label{appx:pixeldp}
%Figure  \ref{fig:compare} compares the privacy–utility trade-off between PixelDP and Q-ShiftDP on Bars \& Stripes. Figure  \ref{fig:comparerest} further shows results on Downscaled MNIST and Binary Blobs. Overall, Q-ShiftDP achieves better utility preservation under the same privacy guarantees compared to to PixelDP.

\subsection{Performance of \NAMEB}
In this section, we provide experimental validation for our adaptive algorithm, \textbf{\NAMEB}. The central premise behind \NAMEB is that the non-adaptive \NAMEA algorithm relies on a conservative, worst-case lower bound for the shot-noise variance ($\sigma^2_{\text{shot}}$), which can be trivially zero or based on a small depolarizing noise factor $\alpha$. We hypothesize that the empirically measured variance is, in practice, significantly higher.
To demonstrate the practical benefits of exploiting this, our evaluation is structured in three parts. First, we visualize the distribution of empirically measured shot-noise variances to confirm that they are consistently greater than the theoretical lower bound. Then, we quantify the resulting reduction in the required artificial noise ($\sigma^2_B$) when using our adaptive estimator. Finally, we compare the end-to-end performance of \NAMEB against the non-adaptive \NAMEA in various depolarizing noise settings, showing that this reduction in noise translates directly to improved utility.

To begin, we set up experiments to measure the distribution of empirically observed shot-noise variances. We consider a PQC with $n = 4$ qubits and 12 parameterized gates. For each of $1000$ random samples, consisting of input $x$ and parameters $\theta$, we simulate the exact single-shot variance according to Eq.~13. The simulations are performed with two observables: a binary observable with $2$ eigenvalues, which is also used throughout our experiments on classification tasks, and a non-degenerate observable with $2^n$ eigenvalues, which is commonly used in regression tasks. Figure~\ref{fig:adaptive:statistics:binary} shows the statistics of practical single-shot variances for the binary observable, where we observe that the practical variances remain far from the global variance lower bound for all values of $\alpha$. The same trend is observed in Figure~\ref{fig:adaptive:statistics:nondegenerative}, which reports the statistics for the non-degenerate observable. These results highlight the necessity of estimating practical variances to reduce the amount of required artificial noise.

Next, we quantify the practical benefit of the adaptive mechanism by measuring the reduction in artificial noise it achieves. For each batch, we define the $\text{Artificial Noise Reduction} (\%)$ as the fraction of the baseline artificial noise that \NAMEB is able to eliminate by leveraging the inherent shot noise:
$$
\text{Artificial Noise Reduction(\%)} = \frac{C_{\text{DP}}-\sigma_B^2}{C_{\text{DP}}} \times 100
$$
Figure~\ref{fig:noise_reduction} plots the average reduction across various numbers of shots ($N_s \in \{10^2, 10^3, 10^4\}$) and batch sizes ($B \in \{64, 128, 256, 512\}$).  First, the reduction is most dramatic when the number of shots is low. With only $100$ shots, the inherent randomness from measurement is substantial, allowing the algorithm to reduce the required artificial noise by up to 16\%. As the number of shots increases to $10^3$ and $10^4$, the gradient estimates become more precise (less shot noise), diminishing the contribution from inherent randomness. Consequently, the noise reduction drops to around $4\%$ and less than $1\%$, respectively. Second, the relative impact of shot noise is greater for smaller batches. For instance, with $100$ shots, a batch of $64$ achieves a $16\%$ reduction, whereas a larger batch of $512$ sees a smaller reduction of less than $12\%$.

Finally, we demonstrate the practical efficiency of \NAMEB\ by comparing its performance with \NAMEA\ under varying environmental conditions, represented by the strength of depolarizing noise $\alpha \in \{0, 0.1, 0.2\}$. For a fair comparison, both methods are trained with $\varepsilon = 1$, $\delta = \frac{1}{|S|}$ with $|S|$ is the dataset's size, and a small $\beta = 10^{-5} \ll \delta$, ensuring an equivalent level of differential privacy. The QML model used in this experiment consists of a single layer and is trained on the Bars and Stripes dataset ($|S| = 1000$) with $1000$ shots and a batch size of $512$. Figure~\ref{fig:acc_vs_alpha_bar} presents the comparison results. Specifically, it shows that across all three noise settings, \NAMEB\ achieves approximately $10\%$ higher accuracy than \NAMEA. These results highlight the effectiveness of reducing artificial noise in improving the practical performance of differentially private QML models.

\begin{figure}[t]
    \centering

    % First row: three result figures
    \begin{subfigure}[b]{0.32\textwidth}
        \centering
        \includegraphics[width=\textwidth]{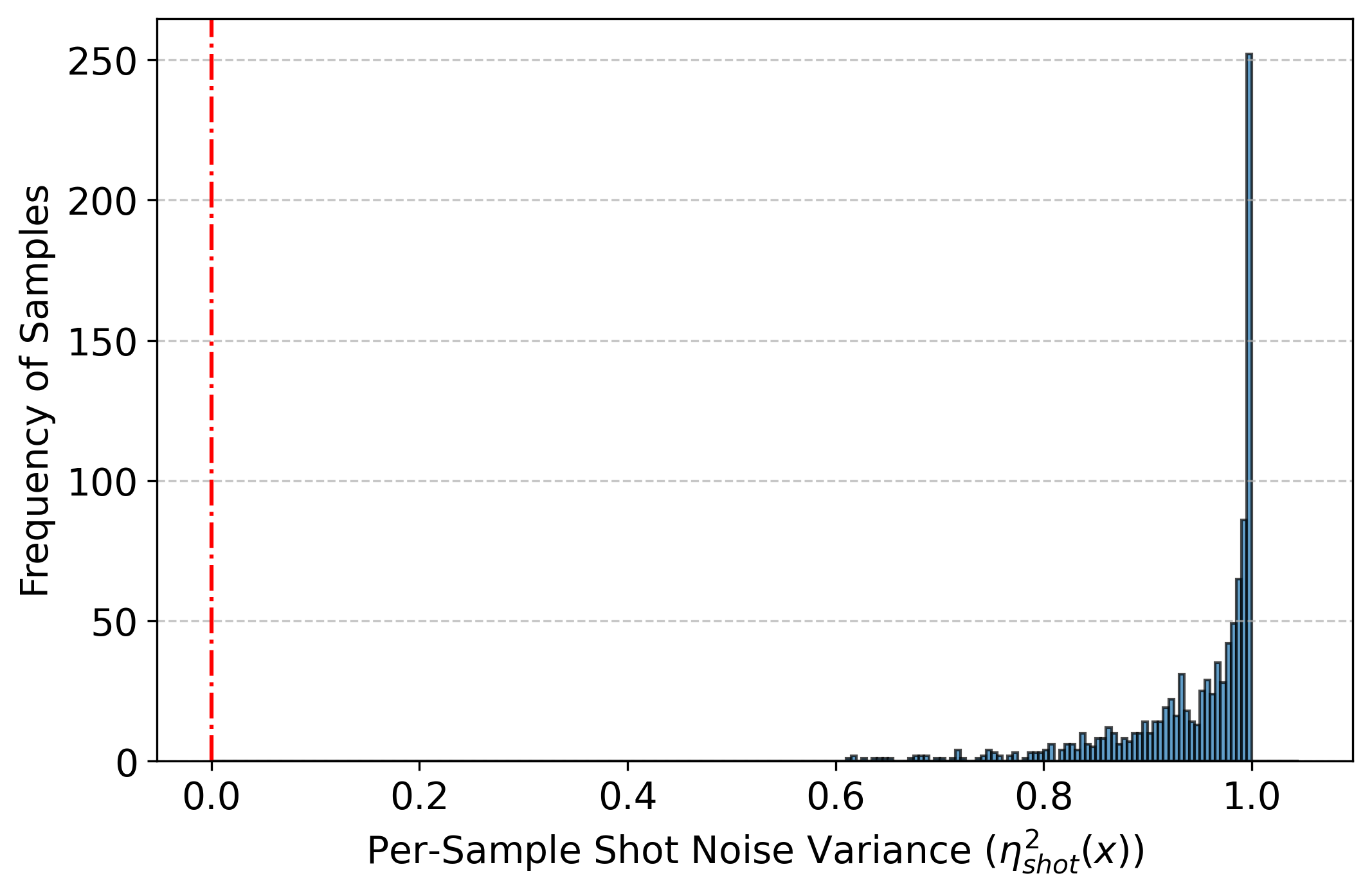}
        %\caption{Result 1} % optional 
        \caption{$\alpha = 0.0$}
    \end{subfigure}
    \hfill
    \begin{subfigure}[b]{0.32\textwidth}
        \centering
        \includegraphics[width=\textwidth]{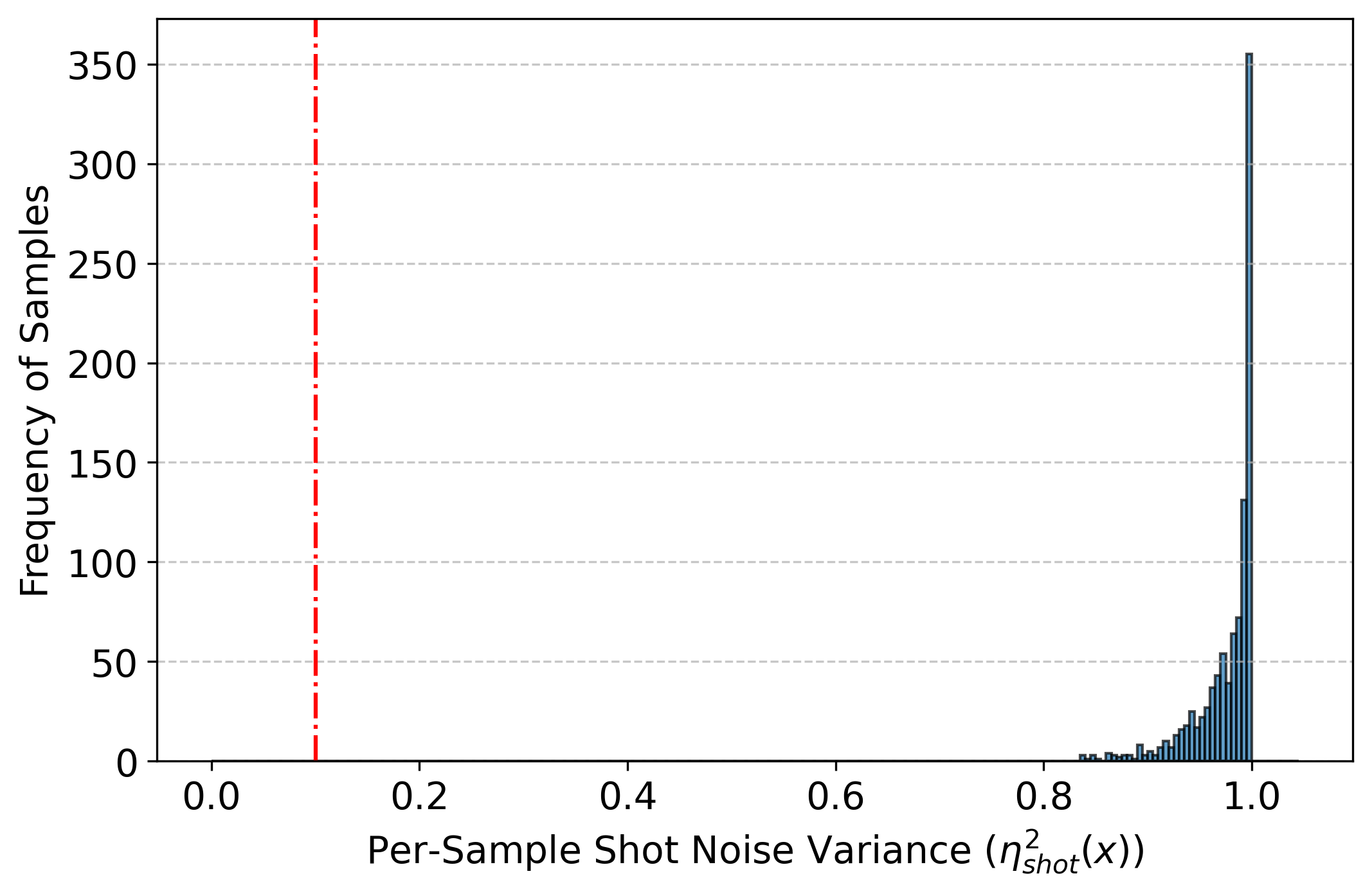}
        \caption{$\alpha = 0.1$}
    \end{subfigure}
    \hfill
    \begin{subfigure}[b]{0.32\textwidth}
        \centering
        \includegraphics[width=\textwidth]{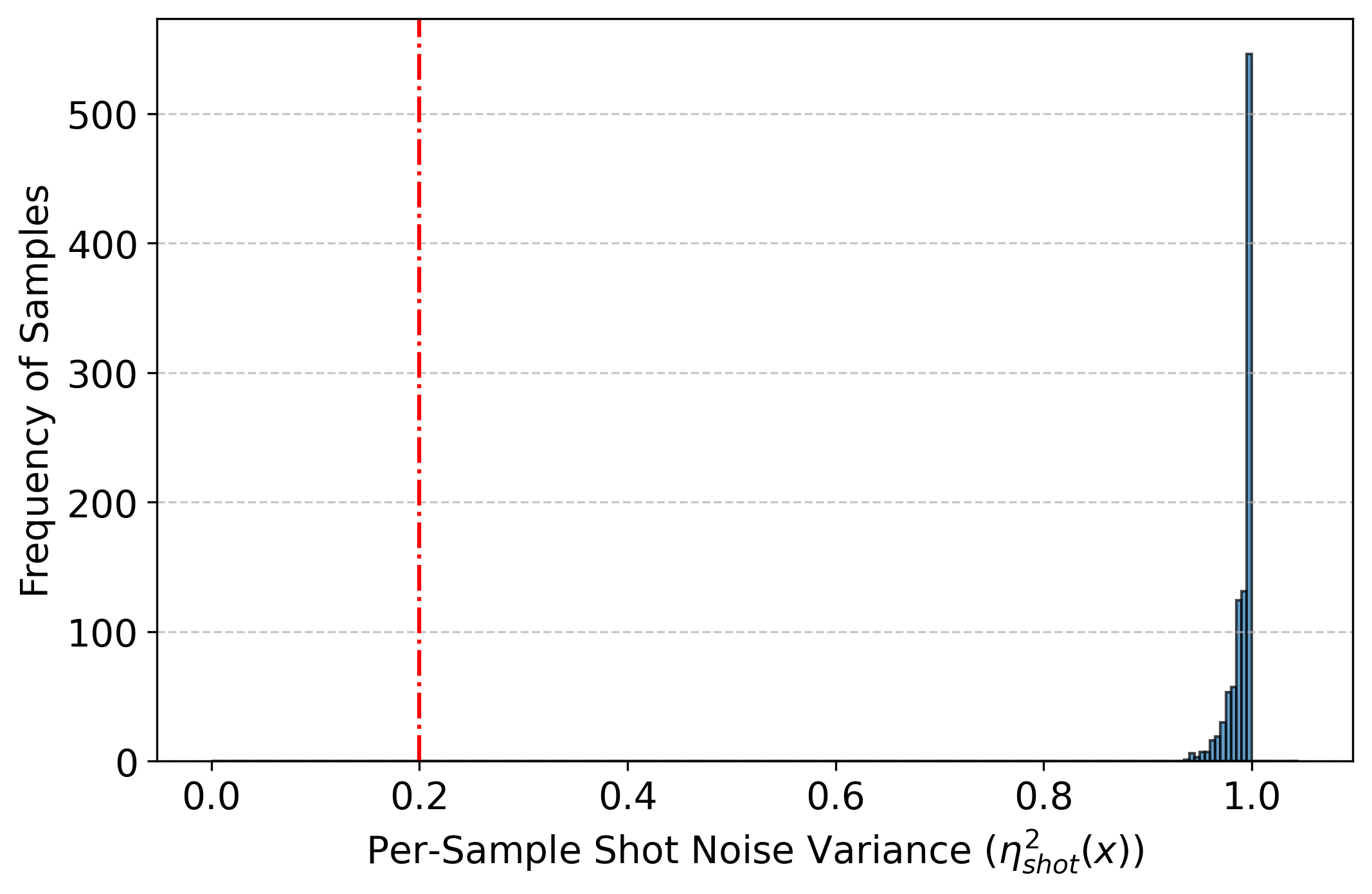}
        %\caption{Result 3}
        \caption{$\alpha = 0.2$}
    \end{subfigure}

    % Second row: legend
    \vspace{0mm} % adjust space if needed
    \begin{subfigure}[b]{0.4\textwidth}
        \centering
        \includegraphics[width=\textwidth]{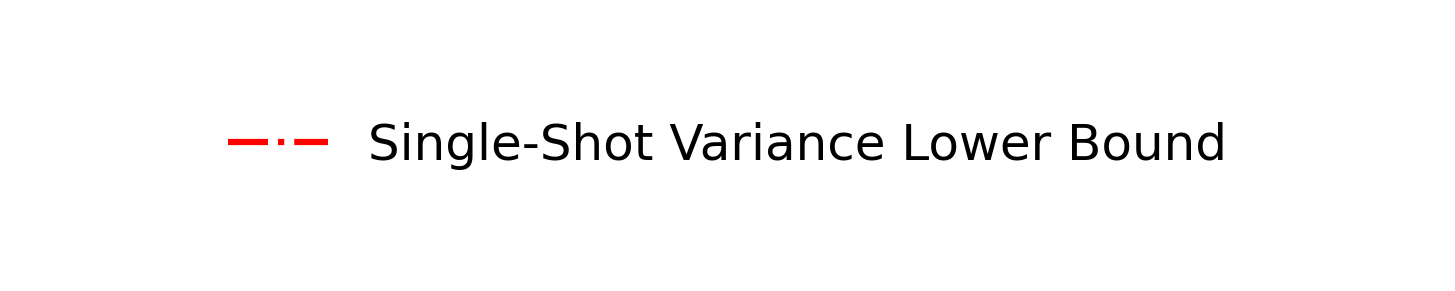}
        %\caption{Legend} % usually we don't caption legend
    \end{subfigure}

    % Main caption for the figure
    \caption{Statistics of Practical Single-Shot Variances in Parameterized Quantum Circuits with a Binary Observable for Depolarizing Noise Levels $\alpha \in [0.0,0.1,0.2]$.}
    \label{fig:adaptive:statistics:binary}
\end{figure}

\begin{figure}[t]
    \centering

    % First row: three result figures
    \begin{subfigure}[b]{0.32\textwidth}
        \centering
        \includegraphics[width=\textwidth]{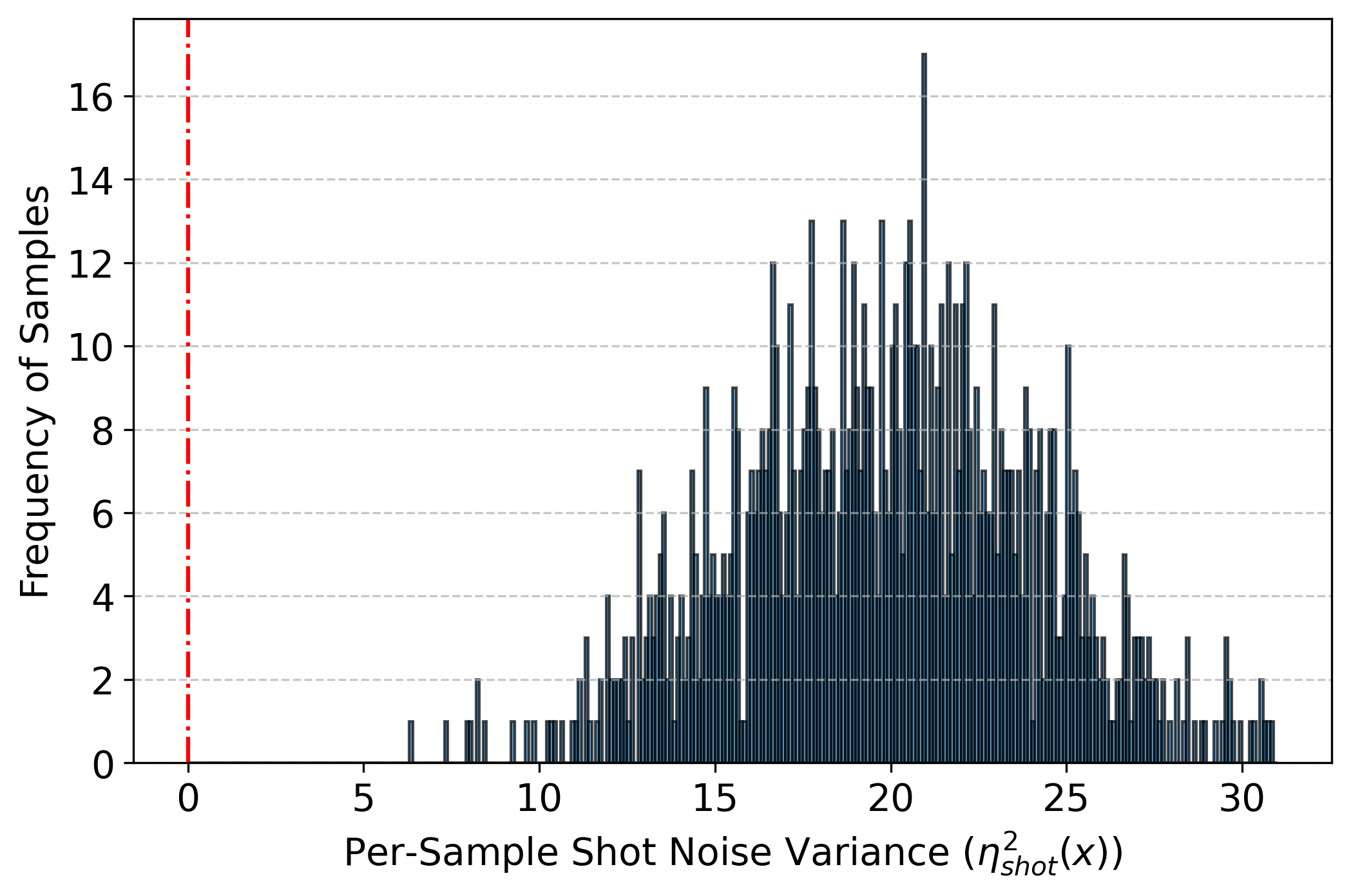}
        %\caption{Result 1} % optional 
        \caption{$\alpha = 0.0$}
    \end{subfigure}
    \hfill
    \begin{subfigure}[b]{0.32\textwidth}
        \centering
        \includegraphics[width=\textwidth]{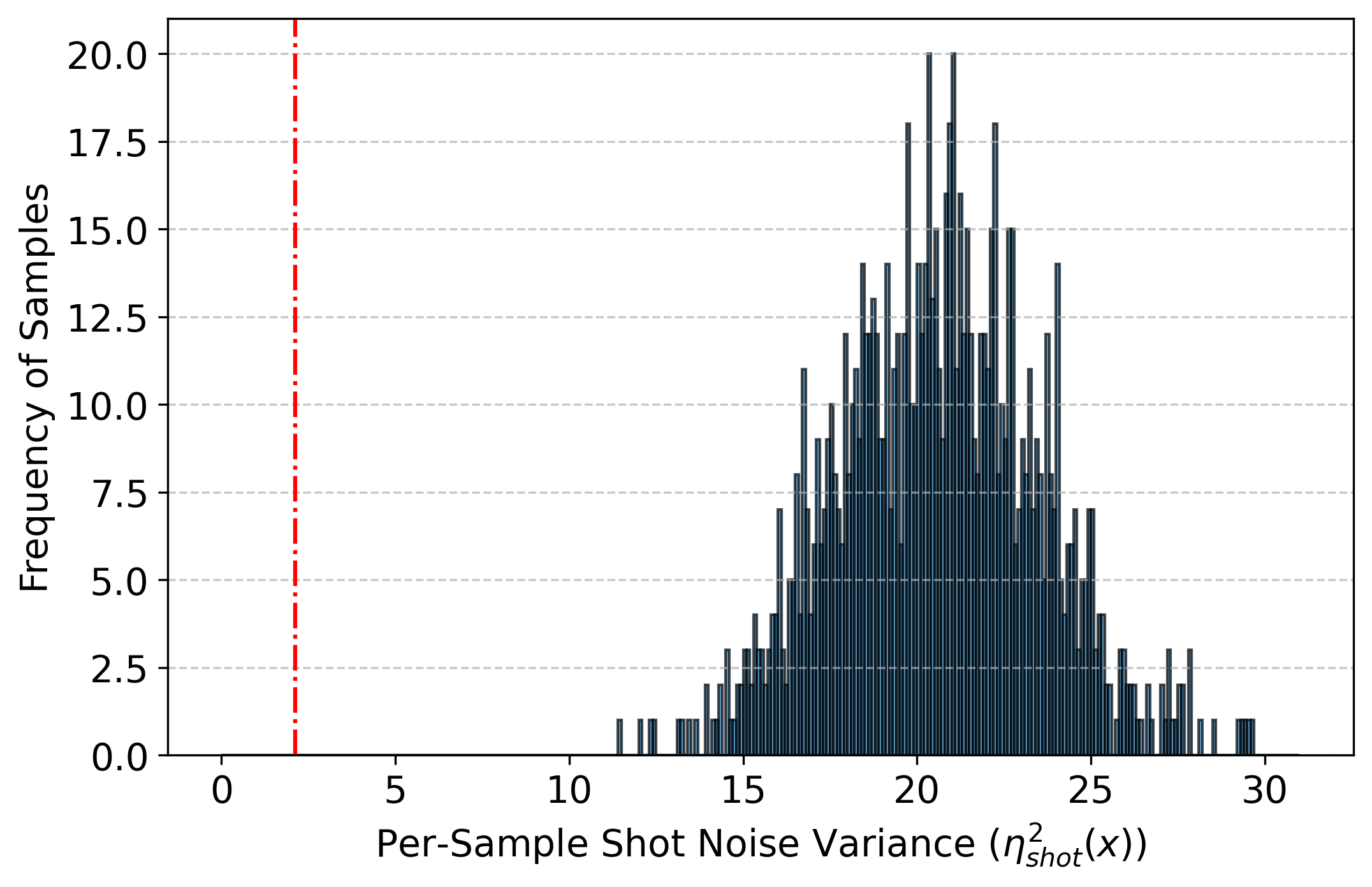}
        \caption{$\alpha = 0.1$}
    \end{subfigure}
    \hfill
    \begin{subfigure}[b]{0.32\textwidth}
        \centering
        \includegraphics[width=\textwidth]{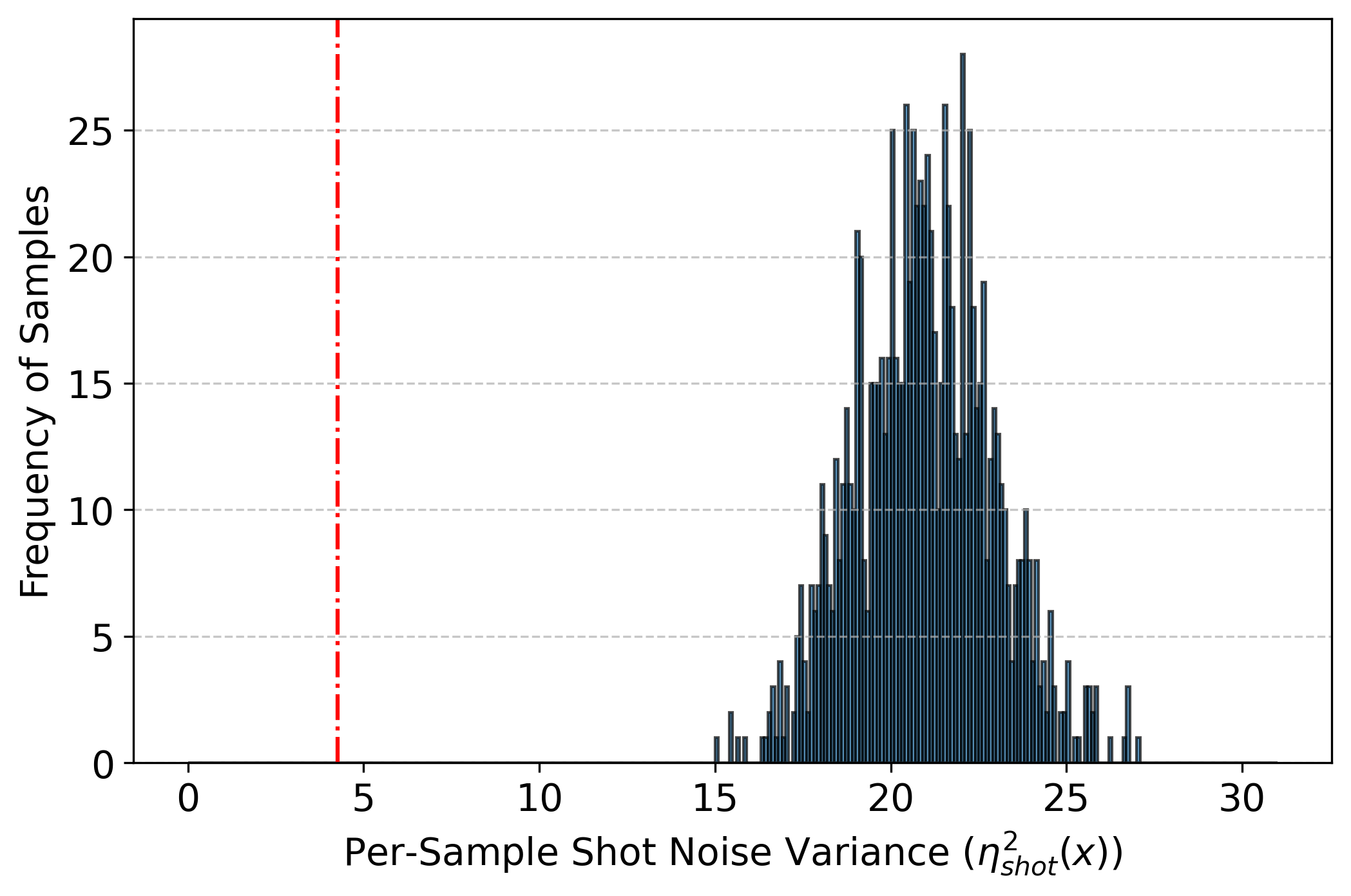}
        %\caption{Result 3}
        \caption{$\alpha = 0.2$}
    \end{subfigure}

    % Second row: legend
    \vspace{0mm} % adjust space if needed
    \begin{subfigure}[b]{0.4\textwidth}
        \centering
        \includegraphics[width=\textwidth]{figs/legend_only.png}
        %\caption{Legend} % usually we don't caption legend
    \end{subfigure}

    % Main caption for the figure
    \caption{Statistics of Practical Single-Shot Variances in Parameterized Quantum Circuits with a Non-Degenerate Observable for Depolarizing Noise Levels $\alpha \in [0.0,0.1,0.2]$.}
    \label{fig:adaptive:statistics:nondegenerative}
\end{figure}

% First row: three result figures
\begin{figure}[t]
    \centering
    \includegraphics[width=.5\textwidth]{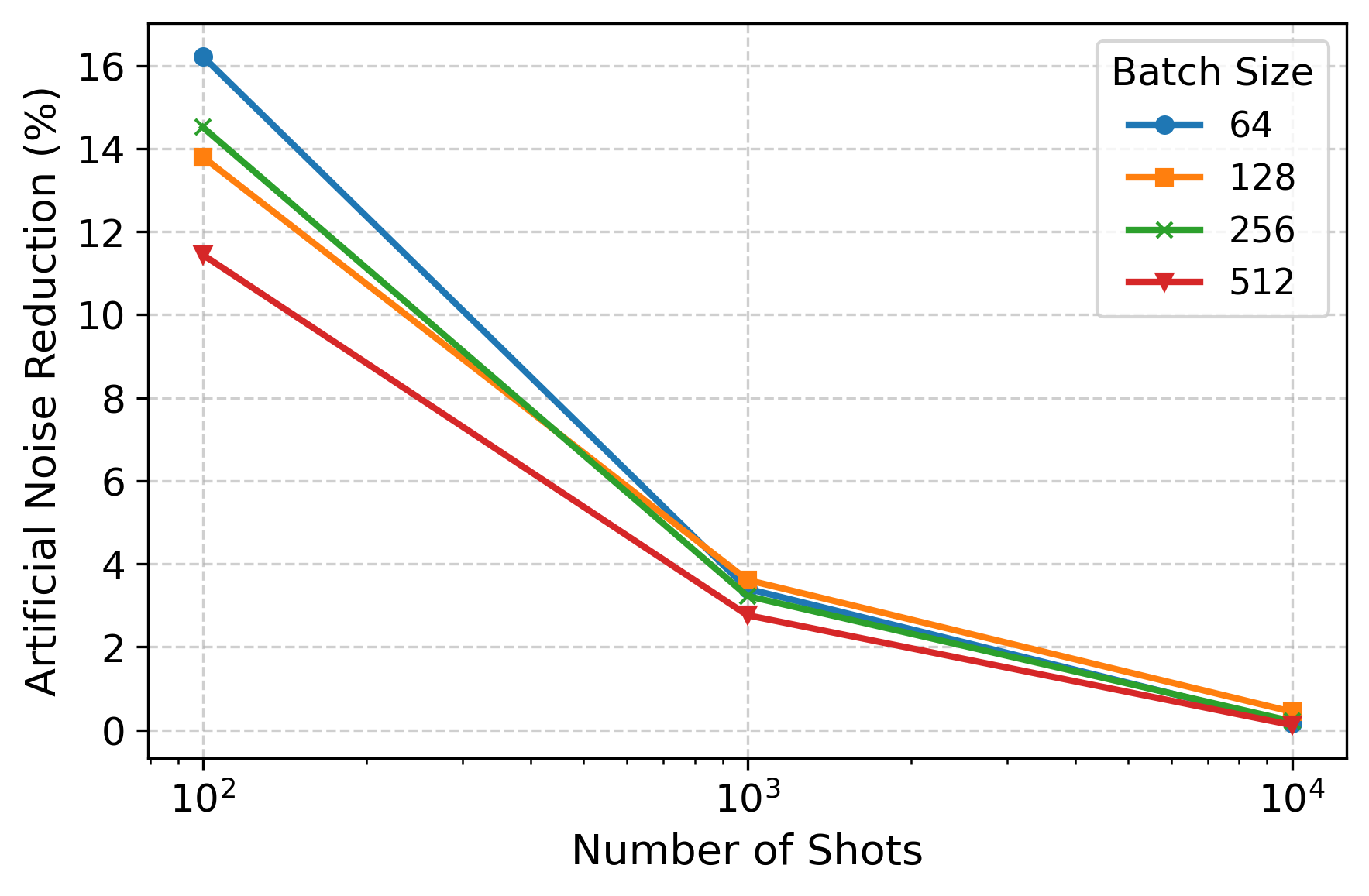}
    \caption{Percentage reduction in required artificial noise achieved by \NAMEB as a function of the number of shots ($N_s$) and batch size ($B$).}
        \label{fig:noise_reduction}

\end{figure}

\begin{figure}[h]
    \centering
    \includegraphics[width=0.7\linewidth]{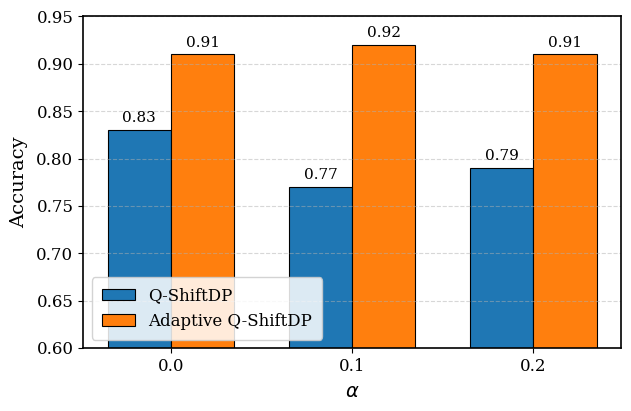}
    \caption{
        Accuracy comparison between Q-ShiftDP and Adaptive Q-ShiftDP 
        under varying $\alpha$ values.
    }
    \label{fig:acc_vs_alpha_bar}
\end{figure}

\section{Estimation of Single-Shot Measurement Variance by Samples (more details)}
\label{sec:add_sampling}

This section provides a detailed walkthrough of our adaptive algorithm, \NAMEB, presented in Algorithm~\ref{alg:adaptive-dp-prs}. \NAMEB improves upon \NAMEA by replacing the fixed, global artificial noise multiplier, $\sigma^2$, with a per-batch multiplier, $\sigma^2_B$, that adapts to the inherent randomness of the specific samples in each mini-batch (All changes from \NAMEA\ are highlighted in \textcolor{red}{red}).
This adaptation is driven by the per-batch variance estimator, $\hat{\bm{\eta}}_B^2$. This estimator, which provides a probabilistic lower bound on the true sum of variances in the batch, is constructed using Lemma~\ref{lemma:pivot_quantitity}. It, in turn, is built from two key statistics computed for each sample $j$ and weight $k$ in the batch: the sample variances, $\bar{\eta}^2_{j,k,\pm}$, and the sample fourth central moments, $\bar{\mu}_{4,j,k,\pm}$.

These statistics are calculated directly from the measurement outcomes. For the positive-shifted circuit of sample $j$ with outcomes $\{\hat{a}_{j,k,1}, \dots, \hat{a}_{j,k,N_s}\}$, we first compute the sample mean, $\bar{a}_{j,k} = \frac{1}{N_s}\sum_i \hat{a}_{j,k,i}$. The estimators are then formulated as follows:
\begin{itemize}
    \item \textbf{Sample Variance ($\bar{\eta}^2_{j,k,+}$):}
    \begin{equation}
    \bar{\eta}^2_{j,k,+} = \frac{1}{N_s - 1} \sum_{i=1}^{N_s} (\hat{a}_{j,k,i} - \bar{a}_{j,k})^2
    \label{eq:sample_variance}
    \end{equation}
    \item \textbf{Sample Fourth Central Moment ($\bar{\mu}_{4,j,k,+}$):}
    \begin{equation}
    \bar{\mu}_{4,j,k,+} = \frac{1}{N_s} \sum_{i=1}^{N_s} (\hat{a}_{j,k,i} - \bar{a}_{j,k})^4
    \label{eq:sample_forth}
    \end{equation}
\end{itemize}
The same formulations are used for the negative-shifted circuit's outcomes, $\{\hat{b}_{j,k,i}\}_{i=1}^{N_s}$, to compute $\bar{\eta}^2_{j,k,-}$ and $\bar{\mu}_{4,j,k,-}$.

In the main manuscript, we state that given a DP budget $(\varepsilon, \delta)$ and assuming identical gate frequencies $\Omega$, \NAMEA can achieve $(\varepsilon, \delta)$-DP guarantee by assigning a batch-specific noise level $\sigma_B$ such that:
\begin{equation}
    \sigma^{2}_B \ge  \max\left(0, C_{\text{DP}} - \frac{\Omega^2\bm{\eta}_{B}^2}{4N_s \Delta^2} \right)
\end{equation}
where $\bm{\eta}_{B}^2$ is the true per-batch variance of the batch $B$. We explain this claim in more detail here. First, we remind that the privacy guarantee of the algorithm stems from the total effective noise variance per iteration (i.e., per batch), which must meet the threshold $C_{\text{DP}}$ (see proof of Theorem~\ref{thm:appx:total_privacy}). This total noise is the sum of two independent sources: the variance from the artificial noise we inject ($\sigma^2_B\Delta/B$) and the variance from the inherent quantum shot noise per batch. For a single gradient component $\mathbf{g}_k^{(j)}$, we have:
$$
\operatorname{Var}[\mathbf{g}_k^{(j)}] = \operatorname{Var}\left[\frac{\Omega}{2}(r_{j,k,+} - r_{j,k,-})\right] = \frac{\Omega^2}{4N_s} \left( \eta^2_{j,k,+} + \eta^2_{j,k,-} \right)
$$
Thus, the shot-noise variance for the entire batch is the sum of the variances of all $B \times K$ components:
\begin{align*}
\operatorname{Var}\left[\sum_{j=1}^B \mathbf{g}^{(j)}\right] &= \sum_{j=1}^B \sum_{k=1}^K \operatorname{Var}[\mathbf{g}_k^{(j)}] \\
&= \frac{\Omega^2}{4N_s} \sum_{j=1}^B \sum_{k=1}^K \left( \eta^2_{j,k,+} + \eta^2_{j,k,-} \right) \\
&= \frac{\Omega^2}{4N_s} \bm{\eta}_{B}^2
\end{align*}

Following the proof of Lemma~\ref{lemma:appx:iterationDP}, we have the total gradient variance per batch as:
$$
\operatorname{Var}[\tilde{\textbf{g}}] = \frac{\Omega^2}{4N_sB^2} \bm{\eta}_{B}^2 + \frac{\sigma^2_B\Delta^2}{B^2}
$$

With the sensitivity $\Delta/B$~\cite{Dwork2006DifferentialP}, we have the total noise multiplier per batch as:
$$
\sigma_{\text{total}}^2 = \sigma_B^2+\frac{\Omega^2\bm{\eta}_{B}^2}{4N_s \Delta^2}
$$
Following the proof of Theorem~\ref{thm:appx:total_privacy},, $(\varepsilon,\delta)$-DP is guaranteed when $\sigma_{\text{total}}^2 \ge C_{DP}$ or equivalently $\sigma^{2}_B \ge  \max\left(0, C_{\text{DP}} - \frac{\Omega^2\bm{\eta}_{B}^2}{4N_s \Delta^2} \right)$.

Another important point is that $\bm{\eta}_{B}^2$ is unknown \emph{a priori}. Therefore, in \NAMEB, we estimate it using $\hat{\bm{\eta}}_{B}^2$. By setting
$$
\sigma_{B}^{2} \leftarrow \max\left(0,\, C_{\text{DP}} - \frac{\Omega^{2}\hat{\bm{\eta}}_{B}^{2}}{4N_{s}\Delta^{2}}\right),
$$
\NAMEB ensures $(\varepsilon,\,(1-\beta)\delta + \beta)$-differential privacy with $\beta$ is the significance level, as stated in Theorem~\ref{thm:appx:advanced_adaptive_privacy}.

\begin{algorithm*}[h]
\DontPrintSemicolon
\caption{Adaptive Differentially Private Parameter-Shift Rule (\NAMEB)}
\label{alg:adaptive-dp-prs}
\KwIn{
    Data encoding $x \mapsto \ket{\psi_0(x)}$,
    Parametric unitary $\hat{U}(\boldsymbol{\theta}) = \prod_{k=1}^{K} e^{i\theta_k G_k}$ with identical frequency $\Omega$,
    Set of observables $\{\hat{O}_y\}$ with eigenvalues in $[\lambda_{\min}, \lambda_{\max}]$,
    Private dataset $S$,
    Learning rate $\textbf{\text{lr}}$, NumSteps $T$, Batch size $B$,
    \textcolor{red}{Ideal DP budget $(\varepsilon, \delta)$, significance level $\beta$.}
}
\KwOut{Trained parameters $\boldsymbol{\theta}_{T}$.}
\BlankLine
\tcp{Define sensitivity based on the theoretical maximum L2 norm of the quantum gradient}
$\Delta \leftarrow \frac{\lambda_{\max} - \lambda_{\min}}{2} \sqrt{\sum_{k=1}^{K} \Omega_k^2}$\;

Initialize parameters $\boldsymbol{\theta}_0$ randomly\;
\For{$t \leftarrow 0$ \KwTo $T-1$}{
    Sample a mini-batch $B \subset S$\;
    Initialize an empty list for gradient estimates: $\mathcal{G} \leftarrow []$\;
    \ForEach{sample $(x_j, y_j) \in B$}{
        %\tcp{Compute the quantum gradient vector}
        
        Initialize vector: $\mathbf{g}^{(j)} \leftarrow \mathbf{0} \in \mathbb{R}^K$\;

        %Prepare the observable $\hat{O} = |y_j\rangle \langle y_j|$
        
        \For{$k \leftarrow 1$ \KwTo $K$}{
            %\tcp{Compute k-th component of the gradient, $\partial C / \partial \theta_k$}
            
            Define shift amount: $s \leftarrow \pi / (2\Omega_k)$\;
            
            %\tcp{Define circuit components}
            $|\phi\rangle \leftarrow (\prod_{l=1}^{k-1} \hat{U}_l(\theta_l)) |\psi_0(x_j)\rangle$\ \tcp*{Define intermediate state}
            $\hat{A} \leftarrow (\prod_{l=k+1}^{K} \hat{U}_l(\theta_l))^\dagger \hat{O}_{y_j} (\prod_{l=k+1}^{K} \hat{U}_l(\theta_l))$\ \tcp*{Define intermediate operator}

            %\tcp{Evaluate shifted expectation values}
            %\tcp{Evaluate shifted expectation values by sampling}
            
            Run positive-shifted circuit $\hat{A} \hat{U}_k(\theta_k+s)|\phi\rangle$ $N_s$ times, collect outcomes $\{\hat{a}_{j,k,i}\}_{i=1}^{N_s}$\;
            $r_+ \leftarrow \frac{1}{N_s}\sum_{i=1}^{N_s} \hat{a}_{j,k,i}$\ \tcp*{Estimator for positive shift}
            
            Run negative-shifted circuit $\hat{A} \hat{U}_k(\theta_k-s)|\phi\rangle$ $N_s$ times, collect outcomes $\{\hat{b}_{j,k,i}\}_{i=1}^{N_s}$\;
            $r_- \leftarrow \frac{1}{N_s}\sum_{i=1}^{N_s} \hat{b}_{j,k,i}$\ \tcp*{Estimator for negative shift}

            \textcolor{red}{Estimate sample variances $\bar{\eta}^2_{j,k,\pm}$ and sample fourth central moments $\bar{\mu}_{4,j,k,\pm}$. from outcomes using Eq~\ref{eq:sample_variance} and \ref{eq:sample_forth}.}
            %$r_+ \leftarrow \langle\phi|\hat{U}_k(\theta_k+s)^\dagger \hat{A} \hat{U}_k(\theta_k+s)|\phi\rangle$\ \tcp*{Estimate with $N_s$ shots}
            
            %$r_- \leftarrow \langle\phi|\hat{U}_k(\theta_k-s)^\dagger \hat{A} \hat{U}_k(\theta_k-s)|\phi\rangle$\ \tcp*{Estimate with $N_s$ shots}
            
            $\mathbf{g}^{(j)}_k \leftarrow \frac{\Omega_k}{2} (r_+ - r_-)$ \tcp*{Parameter-shift rule estimator}

            %$\mathbf{g}^{(j)}_k \leftarrow \bar{\mathbf{g}}^{(j)}_k / \max\left(1, \frac{|\bar{\mathbf{g}}^{(j)}_k|}{C_{clip}}\right)$\ \tcp*{Clip the estimator}
            
        }
        %\tcp{Clip the L2 norm of the gradient estimate}
        %$\bar{\mathbf{g}}_j \leftarrow \mathbf{g}_j / \max\left(1, \frac{\|\mathbf{g}_j\|_2}{C_{clip}}\right)$\;
        Append $\mathbf{g}^{(j)}$ to $\mathcal{G}$\;
    }
    %\tcp{Aggregate, add noise, and update}
    $\hat{\mathbf{g}}_{B} \leftarrow \sum_{\mathbf{g}^{(j)} \in \mathcal{G}} \mathbf{g}^{(j)}$\;
    \tcp{--- Adaptive Step (Per-Batch) ---}
    \textcolor{red}{$\bar{\bm{\eta}}_B^2 \leftarrow \sum_{j=1}^B \sum_{k=1}^K \sum_{\tau\in \{+,-\}} \bar{\eta}^2_{i,\tau}$}

    \textcolor{red}{$\hat{\bm{\eta}}_B^2 \leftarrow \bar{\bm{\eta}}_B^2 - z_\beta\sqrt{\sum_{j=1}^B \sum_{k=1}^K \sum_{\tau\in \pm} \frac{1}{N_s}\left(\bar{\mu}_{4,j,k,\tau} - (\bar{\eta}_{j,k,\tau}^2)^2\right)}$}\;

    \textcolor{red}{$C_{\text{DP}} \leftarrow \left( c_2 \frac{q \sqrt{T \log(1/\delta)}}{\varepsilon} \right)^2$}\;
    
    \textcolor{red}{$\sigma^2_B \leftarrow \max\left(0, C_{\text{DP}} - \frac{\Omega^2\hat{\bm{\eta}}_{B}^2}{4N_s \Delta^2} \right)$}
    
    \textcolor{red}{Sample noise: $\mathbf{z} \sim \mathcal{N}(0, \sigma^2_B \Delta^2 \mathbf{I}_K)$}\;
    $\tilde{\mathbf{g}} \leftarrow \frac{1}{B} \left( \hat{\mathbf{g}}_{B} + \mathbf{z} \right)$\;
    $\boldsymbol{\theta}_{t+1} \leftarrow \boldsymbol{\theta}_t - \textbf{\text{lr}} \cdot \tilde{\mathbf{g}}$\;
}
\Return $\boldsymbol{\theta}_T$
    \end{algorithm*}

\end{document}